\documentclass{article}
\usepackage[margin=1.5in]{geometry}
\usepackage[utf8]{inputenc} 
\usepackage[T1]{fontenc}    
\usepackage{hyperref}       
\usepackage{url}            
\usepackage{booktabs}       
\usepackage{amsfonts}       
\usepackage{nicefrac}       
\usepackage{microtype}      
\usepackage{amsfonts}
\usepackage{amsmath,amsthm}
\usepackage{amssymb,mathrsfs}
\usepackage{enumerate}
\usepackage{algorithm}
\usepackage[noend]{algpseudocode}
\usepackage{algorithmicx}
\usepackage[noend]{algpseudocode}
\usepackage{graphicx}
\usepackage{wrapfig}
\usepackage{mathtools}
\usepackage[stable]{footmisc}
\DeclareMathOperator*{\argmax}{arg\,max}
\DeclareMathOperator*{\argmin}{arg\,min}

\usepackage{tikz}
\usetikzlibrary{arrows,positioning,fit,backgrounds}
\usetikzlibrary{calc}

\usepackage{authblk}

\usetikzlibrary{arrows,positioning,fit,backgrounds,shapes}
\usetikzlibrary{calc}

\newtheorem{definition}{Definition}

\newtheorem{theorem}{Theorem}

\newtheorem{lemma}{Lemma}

\newtheorem{proposition}{Proposition}

\newtheorem{example}{Example}

\newcommand{\be}{\begin{equation}}
\newcommand{\ee}{\end{equation}}
\newcommand{\ben}{\begin{enumerate}}
\newcommand{\een}{\end{enumerate}}
\newcommand{\bea}{\begin{eqnarray}}
\newcommand{\eea}{\end{eqnarray}}
\newcommand{\bean}{\begin{eqnarray*}}
\newcommand{\eean}{\end{eqnarray*}}

\newcommand{\E}{\mathbb{E}}

\newcommand{\gen}{\mathrm{gen}}

\newcommand{\risk}[0]{\mathrm{risk}}
\newcommand{\W}[0]{\mathcal{W}}

\newcommand{\f}{\mathbf{T}}
\newcommand{\dw}{\mathrm{d}w}
\newcommand{\intw}{\int_{w\in\mathcal{A}}}

\title{\bf Maximum Multiscale Entropy and \\ Neural Network Regularization}

\author[1]{Amir R. Asadi\footnote{\href{mailto:aasadi@princeton.edu}{aasadi@princeton.edu}}}
\affil[1]{Department of Electrical Engineering, Princeton University}

\author[2]{Emmanuel Abbe\footnote{\href{mailto:emmanuel.abbe@epfl.ch}{emmanuel.abbe@epfl.ch}}}
\affil[2]{Mathematics Institute and School of Computer and Communication Sciences, EPFL}

\date{June 25, 2020}

\begin{document}

\maketitle

\begin{abstract}
  A well-known result across information theory, machine learning, and statistical physics shows that the maximum entropy distribution under a mean constraint has an exponential form called the Gibbs--Boltzmann distribution. This is used for instance in density estimation or to achieve excess risk bounds derived from single-scale entropy regularizers (Xu--Raginsky '17). This paper investigates a generalization of these results to a multiscale setting. We present different ways of generalizing the maximum entropy result by incorporating the notion of scale. For different entropies and arbitrary scale transformations, it is shown that the distribution maximizing a multiscale entropy is characterized by a procedure which has an analogy to the renormalization group procedure in statistical physics. For the case of decimation transformation, it is further shown that this distribution is Gaussian whenever the optimal single-scale distribution is Gaussian. This is then applied to neural networks, and it is shown that in a teacher-student scenario, the multiscale Gibbs posterior can achieve a smaller excess risk than the single-scale Gibbs posterior. 
\end{abstract}

\section{Introduction}\label{Introduction section} 
Many real-world signals and physical systems have a large variety of length scales in their structures. 
Such multiscale structures have been studied and exploited on different fields and with different tools, such as in statistical mechanics, Gaussian processes, and signal processing among others. In this paper, we study distributions which maximize uncertainty at different scales simultaneously under a mean constraint.
We are in particular motivated by the goal of analyzing neural networks and their generalization error while exploiting their multilevel characteristic. 

The central notion studied in this paper, the \emph{multiscale entropy}, is simply defined by taking a linear mixture of entropies of a system at different length scales. Hence, the multiscale entropy is a generalization of the classical entropy. For instance, if $(W_1,W_2,\dots,W_d)$ denote the layers of a neural network, and for two distributions $P_{W_1\dots W_d}$ and $Q_{W_1\dots W_d}$ defined on the weight parameters, the  multiscale relative entropy between $P$ and $Q$ with index $\sigma=(\sigma_1,\dots,\sigma_d)$ is given by: 
\begin{align}
    &D_{(\sigma)}(P_{W_1\dots W_d}\|Q_{W_1\dots W_d})\nonumber\\& = \sigma_1 D(P_{W_1\dots W_d}\|Q_{W_1\dots W_d}) + \sigma_2 D(P_{W_1\dots W_{d-1}}\|Q_{W_1\dots W_{d-1}})+\dots+\sigma_d D(P_{W_1}\|Q_{W_1}).\label{multiscale entropy example}
\end{align}
Notice that $D_{(\sigma)}(P_{W_1\dots W_d}\|Q_{W_1\dots W_d})$ is a linear mixture of entropies each at different scales of the network, where the scale in this example is coupled to the depth of the network. Also note that classical relative entropy corresponding to the whole system $D(P_{W_1\dots W_d}\|Q_{W_1\dots W_d})$ is a special case of $D_{(\sigma)}(P_{W_1\dots W_d}\|Q_{W_1\dots W_d})$ when $\sigma=(1,0,\dots,0)$.
In this paper, we study the optimization of such entropies under a mean constraint, namely:
\begin{equation}
    \argmin_{P_{W_1\dots W_d}:\E[f(W_1,\dots,W_d)]=\mu}D_{(\sigma)}(P_{W_1\dots W_d}\|Q_{W_1\dots W_d}). \label{mme}
\end{equation}  
We refer to \eqref{mme} as the {\it minimum multiscale relative entropy}, and we sometimes also view this as part of a family of {\it maximum multiscale entropy} problems, where the `entropy' corresponds in this case to $-D$. 
This is a generalization of the widely known maximum entropy problem. For example, it is a well-known result in information theory that the minimizing distribution $P$ of 
\begin{equation}\label{Gibbs minimization problem}
	\E[f(W)]+\lambda D(P_W\|Q_W), 
\end{equation}
is the Gibbs--Boltzmann distribution $P(w)\propto \exp(-\frac{f(w)}{\lambda})Q(w)$,
where $Q$ is a fixed distribution, $W$ is distributed according to $P$, $f$ is a measurable function called the energy function, and $\lambda>0$. 
This fact\footnote{Notice that this is the Lagrangian of the problem of minimizing relative entropy $D(P_W\|Q_W)$ under the mean constraint $\E[f(W)]=\mu$.} dates back to the work of Jaynes \cite{jaynes1957information} on the maximum entropy inference, and was revisited\footnote{This was also generalized using different entropies such as Tallis entropy, R\'{e}nyi entropy and others; see \cite{peterson2013maximum} and references therein. 
} 
in a broader context by Csisz\'{a}r under the property that exponential families achieve the I-projection over linear families \cite{csiszar1975divergence}. This property has diverse important applications, such as in the celebrated papers on species distribution modeling by Phillips et al. \cite{phillips2004maximum} and natural language processing by Berger et al. \cite{berger1996maximum}, as well as in statistical mechanics \cite{jaynes1957information, jaynes1957information2}. 

This result also has a concrete application in the context of statistical learning theory. 
The empirical risk can be written as $\E[f(W)]$ where $W$, the output of the learning algorithm, depends on the sample distribution $P_S$ and $f$ depends on the loss function. The generalization error of an hypothesis generated under the algorithm/channel $P_{W|S}$ can then be bounded under mild assumptions using the KL-divergence $D(P_{S}P_{W|S} \| P_{S}P_{W})$ in PAC-Bayes and mutual information-based bounds \cite{catoni2007pac,Russo}. Therefore, the {\it minimum single-scale relative entropy} problem in \eqref{Gibbs minimization problem} becomes precisely an upper-bound on the population risk and its minimizer gives precisely the distribution under which one should sample from the hypothesis set to minimize this bound \cite{xu2017information}.

As mentioned, this paper studies multiscale versions of the previous two paragraphs. 
This is motivated by the multilevel nature of neural networks, where  
the set of all mappings between the input and each hidden layer is a refinement of the hypothesis set at different scales, each scale corresponding to the depth of the hidden layer.
It has recently been shown in \cite{asadi2018chaining,asadi2020chainingchainrule} that the notion of scale can be employed to further exploit the closeness and similarity among the hypotheses of a learning model to tighten the generalization bounds from \cite{Russo, xu2017information}. This is achieved by replacing the mutual information bound $D(P_{S}P_{W|S} \| P_{S}P_{W})$ with a sum of mutual informations that each consider the hypothesis set at different scales, as discussed in Section \ref{neural network section}. 

In the simplest case of two scales, this approach brings to light the following  generalization of (\ref{Gibbs minimization problem}): finding the minimizing distribution $P_{W_1W_2}$ of   
\begin{align}
	\E &[f(W_1,W_2)]+\lambda_1 D(P_{W_1W_2}\|Q_{W_1W_2})+\lambda_2 D(P_{W_1}\|Q_{W_1}),\label{Biscale Gibbs minimization problem}
\end{align}
where $(W_1,W_2)\sim P_{W_1W_2}$ and $\lambda_1,\lambda_2>0$. Notice how the regularizer has the form of a multiscale entropy as given in (\ref{multiscale entropy example}).
Here we assume that $W_1$ is a vector of random variables lying at the coarser scale, 
and $W_2$ is a vector of the rest of the random variables. Therefore at the finer scale, we observe the total variables $(W_1,W_2)$, and at the coarser scale, we only observe $W_1$.
Notice that this problem reduces to (\ref{Gibbs minimization problem}) when $\lambda_2=0$. However, when $\lambda_2>0$, we are amplifying the uncertainty at the coarser scale by taking it into account in both terms. We will next resolve this type of problem in a general context, relating the maximizing procedure to the renormalization group theory from statistical physics \cite{wilson1974renormalization, wilson1979problems}, and describe applications to neural networks. 

\subsection{Contributions of this paper}
\begin{enumerate}[(I)]
	\item We characterize the maximum multiscale distributions for arbitrary scale transformations and for entropies (discrete or continuous) in Theorem \ref{Max multiscale shannon entropy}, as well as for arbitrary scale transformations and relative entropies in Theorem \ref{multiscale relative entropy RG theorem}; we describe in particular how these are obtained from procedures that relate to the renormalization group in statistical physics.
	\item We show in Theorem \ref{Gaussian closedness} that for the special case of decimation scale transformation, which relates to the multilevel structure of neural networks, the optimal multiscale distribution is a multivariate Gaussian whenever the optimal single-scale distribution is multivariate Gaussian; see Section \ref{Gaussians section}. We then use this fact in our simulations in Section \ref{Experiment section} (point IV below).
	\item We demonstrate in Theorem \ref{DPG thoerem} the tightness of the excess risk bound for the multiscale Gibbs posterior over the  classical Gibbs excess risk bound \cite{raginsky2016information}, and provide an example in a teacher-student setting (i.e., data generated from a teacher network with smaller depth and learned by a deeper network) in Subsection \ref{TeacherStudent section}.
	\item We show in Subsection \ref{Experiment section} how the multiscale Gibbs posterior encompasses both the classical Gibbs posterior and the random feature training as special cases, and provide a simulation showing how the multiscale version improves on the two special cases in the teacher-student setting.
\end{enumerate}

\subsection{Further relations with prior work}
    \paragraph{\bf PAC-Bayes generalization bounds.} 
    The generalization bound used in this paper has commonalities with PAC-Bayes bounds, first introduced by \cite{McAllester,mcallester1999pac}, in that one first expresses prior knowledge by defining a prior distribution over the hypothesis class without assuming the truth of the prior. Then an information-theoretic `distance' between the prior distribution and any posterior---with which a hypothesis is randomly chosen---appears in the generalization bound. However, unlike the generic PAC-Bayes bound of \cite{mcallester1999pac}, our generalization bound is multiscale and uses the multiscale relative entropy rather than a single relative entropy for the whole hypothesis set. The motivation is to exploit the compositional structure of neural networks: instead of controlling the `complexity' of the whole network at once, we simultaneously control the added complexities of each layer with respect to its previous layers (i.e., the interactions between the scales). As we show in Theorem \ref{DPG thoerem}, with this multiscale bound we can guarantee a tighter excess risk than single-scale bounds. Variants of PAC-Bayes bounds have later been studied in e.g. \cite{mcallester2003simplified, seeger2002pac, langford2003pac}, and have been employed more specifically for neural networks in e.g. \cite{dziugaite2017computing, neyshabur2017pac, zhou2018non, dziugaite2018data}, but again these bounds are not multiscale. The paper \cite{audibert2004pac} combines PAC-Bayes bounds with generic chaining and obtains multiscale bounds that rely on auxiliary sample sets, however, an important difference between our generalization bound and \cite{audibert2004pac} is that our bound puts forward the \emph{multiscale entropic regularization} of the empirical risk, for which we can characterize the minimizer exactly. 
    
    \paragraph{\bf Renormalization group and neural networks.} Connections between the renormalization group and neural networks have been pointed out in the seminal works \cite{beny3124deep, mehta2014exact} and later in other papers such as \cite{lin2017does, iso2018scale, koch2018mutual, funai2018thermodynamics, koch2019deep}. These works have mostly focused on applying current techniques in deep learning such as different types of gradient-based algorithms applied on various neural network architectures to problems in statistical mechanics. However, we are employing renormalization group transformation in the other way around, to help with inference with neural networks. Another important difference between our approach and \cite{mehta2014exact} is that these authors perform renormalization group transformations in the \emph{function space} and map the neurons to spins, whereas we perform renormalization group transformations in the \emph{weight space} and map the synapses to spins.  
As a result, in our approach spin decimation does not mean that we ignore some neurons and waste their information. Rather, it simply means that we replace one layer of synapses between two consecutive hidden layers with a fixed \emph{reference mapping} (based on terminology of \cite[Section 4]{asadi2020chainingchainrule}) such as the identity mapping for residual networks, as in Section \ref{neural network section}. 

    \paragraph{\bf Chaining.} Multiscale entropies implicitly appear in the classical chaining technique of high-dimensional probability. For example, one can rewrite Dudley inequality \cite{Dudley} variationally to remove the square root function over the metric entropies and transform the bound into a linear mixture of metric entropies at multiple scales. This is also the case for the information-theoretic extension of chaining with mutual information \cite{asadi2018chaining} which our generalization bound is based upon in Section \ref{neural network section}; see \cite{asadi2020chainingchainrule}.
    
    \paragraph{\bf Approximate Bayesian inference for neural networks.} 
    The recent paper \cite{khan2019approximate} also studies approximate Bayesian inference for neural networks using Gaussian approximations to the posterior distribution, as we similarly do in Subsection \ref{Experiment section} based on Gaussian results of Section \ref{Gaussians section}. However, unlike our approach, their analysis is not multiscale and treats the whole neural network as a single block.
    
     \paragraph{\bf Phase-space complexity.} Between the definition of multiscale entropy in this paper and what papers \cite{zhang1991complexity,fogedby1992phase} refer to as ``phase space complexity'' in statistical mechanics, there exist notional similarities. Characterizing maximum phase space complexity distributions was left as an open question and conjectured to be related to the renormalization group in \cite{fogedby1992phase}.
\section{Multiscale entropies} 
Assume that $W$ is the state of a system or is data, which can be either a random variable or a random vector. We first give the definition of different entropies.
\begin{definition}[Entropy\footnote{All entropies in this paper are in nats.}]
	The Shannon entropy of a distribution $P$ defined on a set $\mathcal{W}$ is 
$
	H(P)=-\sum_{w\in \mathcal{W}} P(w)\log P(w), 
$
if $\W$ is discrete. The differential entropy of $P$ is defined as 
$
	h(P)=-\int_{w\in \mathcal{W}} P(w)\log P(w), 
$
if $\W$ is continuous. The relative entropy between distributions $P_W$ and $Q_W$ defined on the same set $\mathcal{W}$ is 
$
    D(P_W\|Q_W)=\sum_{w\in \mathcal{W}} P_W(w)\log\left(\frac{P_W(w)}{Q_W(w)} \right),
$
if $\mathcal{W}$ is discrete, and 
$
    D(P_W\|Q_W)=\int_{w\in \mathcal{W}} P_W(w)\log\left(\frac{P_W(w)}{Q_W(w)}\right)\mathrm{d}w,
$
if $\mathcal{W}$ is continuous.
\end{definition}
Next, we blend the notions of \emph{scale} and \emph{entropy} as follows:  Let $W^{(1)}\triangleq W$ and given a sequence of \emph{scale transformations} $\mathbf{T}\triangleq\{T_i\}_{i=1}^{d-1}$ assume that $W^{(i+1)}\triangleq T_i(W^{(i)})$ for all $1\leq i\leq d-1$. We define $W^{(i)}$ to be the \emph{scale} $i$ version of the random vector $W$. For a vector of non-negative reals $\mathbf{\sigma}=(\sigma_1,\dots,\sigma_d)$, let $\sigma_i$ denote the \emph{length coefficient} at scale $i$.
\begin{example} In a wavelet theory context, assume that $W$ represents the vector of pixels of an image, and each transformation $T_i$ takes the average value of all neighboring pixels and for each group outputs a single pixel with that average value, thus resulting in an lower resolution image. Hence, here $W^{(2)},\dots, W^{(d)}$ are respectively lower and lower resolution versions of $W$. 
\end{example}
For any given $W$, $\mathbf{T}$, and $\sigma$, we define the multiscale entropies as follows:
\begin{definition}[Multiscale entropy] The multiscale Shannon entropy is defined as 
\begin{equation}
\nonumber
	H_{(\mathbf{\sigma},\mathbf{T})}(W)\triangleq \sum_{i=1}^d {\sigma_i}H(W^{(i)}).
\end{equation}
Let the multiscale differential entropy be 
\begin{equation}
\nonumber
	h_{(\mathbf{\sigma}, \mathbf{T})}(W)\triangleq \sum_{i=1}^d \sigma_i h(W^{(i)}).
\end{equation}
    We define the multiscale relative entropy between distributions $P_W$ and $Q_W$ as
    \begin{equation}
    \nonumber
    	D_{(\sigma, \mathbf{T})}(P_W\|Q_{W}) \triangleq \sum_{i=1}^d \sigma_i D(P_{W^{(i)}}\|Q_{{W}^{(i)}}).
    \end{equation} 
\end{definition}
Notice that multiscale entropy encompasses classical entropy as a special case: it suffices to choose $\sigma = (1,0,\dots,0)$ to get $D_{(\sigma,\mathbf{T})}(P_W\|Q_W)=D(P_{W}\|Q_{W})$ and similarly for the Shannon and differential entropies. However, by taking positive values for $\sigma_i$, $i\geq 2$, we are  
emphasizing the entropy at coarser scales. Next, we focus on a special case of scale transformations called \emph{decimation} which relates with the multilevel structure of neural networks:\footnote{Multiscale relative entropy with decimation transformation is named ``multilevel relative entropy'' in \cite{asadi2020chainingchainrule}.} Assume that $\mathcal{W}=\mathcal{W}_1\times\dots\times \mathcal{W}_d$ and let $W\triangleq (W_1,\dots,W_d)$ denote a random vector partitioned into $d$ vectors. For example, $W$ can denote the synaptic weights of a neural network divided into its layers. For all $1\leq i\leq d$, define $W^{(i)}\triangleq (W_1,\dots,W_{d-i+1})$. Notice that $W^{(1)}=W$  
and larger $i$ gives more random variables in the vector $W$ that we stop observing in $W^{(i)}$. Therefore the scale transformations $\{T_i\}_{i=1}^{d-1}$ simply eliminate the layers one-by-one. In the theoretical physics literature, spin decimation is a type of scale transformation introduced by Kadanoff and Houghton \cite{kadanoff1975numerical} and Migdal \cite{migdal1975recursion}.
\begin{example} The decimation transformation is what is typically used in maps of cities of a certain region. As one zooms out of the region and views the map of the region at larger scales, one omits the smaller cities in the resulting maps.
\end{example}
\section{Maximum multiscale entropy}\label{Maximum multiscale entropy section}
In this section we derive maximum multiscale entropy distributions for different entropies. The key ingredient of the proofs of all derivations is the chain rule of relative entropy.
\subsection{Multiscale Shannon and differential entropy maximization}
Let $f$ be an arbitrary measurable function called the energy. Consider the problem of maximizing Shannon entropy under a mean constraint:
\begin{equation}\label{Max entropy problem}
    \argmax_{P_W: \E[f(W)]=\mu}H(W).
\end{equation}
For this, one solves for the maximizing distribution of the Lagrangian $H(W) -\lambda\mathbb{E}[f(W)]$, which by a well known result due to \cite{jaynes1957information} (see Lemma \ref{Gibbs is the maximizer entropy} in the Appendix), is given by the Gibbs--Boltzmann distribution $\widehat{P}_W(w)\propto \exp\left(-\lambda f(w) \right)$. Now, consider the following multiscale generalization of the previous problem, that is given $f$, $\mu$, $\mathbf{T}$ and $\sigma$, solving for
\begin{equation}
    \nonumber
    \argmax_{P_W: \E[f(W)]=\mu}H_{(\sigma, \mathbf{T})}(W).
\end{equation}
Notice that this problem reduces to (\ref{Max entropy problem}) in the special case when $\sigma_2=\dots = \sigma_{d}=0$. But for more general choices for the values of $\sigma_2,\dots,\sigma_{d}$, the uncertainty at the coarser scales are emphasized.
We form the Lagrangian as follows and define the unconstrained maximization problem:
\begin{equation}\label{Lagrangian Shannon}
	P^*_W\triangleq \argmax_{P_W} \left\{H_{(\sigma,\f )}(W) -\lambda\mathbb{E}[f(W)]\right\}.
\end{equation} 
In the following, for any $\lambda>0$, we find the maximizing distribution $P^{*}_{W}$. First, we require the definition of \emph{scaled distribution}, which is basically raising a probability distribution to a power:

\begin{definition}[Scaled distribution\footnote{This is also known as \emph{escort distribution} in statistical physics literature; see e.g. \cite{bercher2012simple}.}] Let $\theta>0$. If $P$ is a distribution on a discrete set $\mathcal{A}$, then for all $a\in \mathcal{A}$, we define the scaled distribution $(P)_{\theta}$ with
\begin{equation}
	\nonumber
	(P)_{\theta}(a)= \frac{(P(a))^{\theta}}{\sum_{x\in \mathcal{A}}(P(x))^{\theta}}.
\end{equation}
For analog random variables it is defined analogously except by replacing the sum in the denominator with an integral.
\end{definition}
Define the Gibbs distribution at the finest (microscopic) scale as  
\begin{equation}
	P^{\textrm{Gibbs}}_W(w) \triangleq \frac{\exp\left({-\frac{\lambda f(w)}{\sigma_1}}\right)Q(w)}{\sum_w \exp\left({-\frac{\lambda f(w)}{\sigma_1}}\right)Q(w)},
\end{equation}
which would be the minimizer of (\ref{Lagrangian Shannon}) had we had $\sigma_2 = \dots = \sigma_{d}=0$. Algorithm \ref{RC} receives the microscopic Gibbs distribution and the values of $\sigma_2, \dots, \sigma_{d}$ as its input, and outputs the desired distribution $P^{\star}_W$---the maximizer of (\ref{Lagrangian Shannon}).

\begin{algorithm}
    \caption{}
    \label{RC}
    \begin{algorithmic}[1]      
                 \State $U^{(1)}_{W^{(1)}}\gets P^{\rm{Gibbs}}_{W}$ \Comment{Initial microscopic Gibbs distribution}\label{initialization of MT}
                 \For{$i=2 \texttt{ to } d$} 
        	 	 \State $M^{(i)}_{W^{(i)}}\gets T_{i-1}\left(U^{(i-1)}_{W^{(i-1)}}\right)$ \Comment{Coarse-graining }
        	 	 \State $U^{(i)}_{W^{(i)}}\gets \left(M^{(i)}_{W^{(i)}}\right)_{\frac{\sigma_1+\dots+\sigma_{i-1}}{\sigma_{1}+\dots +\sigma_i}}$ \Comment{Renormalization}
     			 \EndFor
            \State \textbf{return} $P^{\star}_{W}= U^{(d)}_{W^{(d)}}U^{(d-1)}_{W^{(d-1)}|W^{(d)}}\dots U^{(1)}_{W^{(1)}|W^{(2)}}$ \Comment{Refinement}
    \end{algorithmic}
\end{algorithm}

For continuous (analog) random variables we replace multiscale Shannon entropy with multiscale differential entropy and consider the following problem:
$
    \max_{P_W: \E[f(W)]}h_{(\mathbf{\sigma},\mathbf{T})}(W).
$ 
We then form the Lagrangian and define 
\begin{equation}\label{Lagrangian differential}
	P^*_W\triangleq \argmax_{P_W} \left\{h_{(\mathbf{\sigma},\mathbf{T})}(W) -\lambda\mathbb{E}[f(W)]\right\}.
\end{equation}
Similarly, Algorithm \ref{RC} outputs $P^{\star}_W$, the maximizer of \eqref{Lagrangian differential}, except now one should define the initial microscopic Gibbs distribution for continuous random variables as 
\begin{equation}
\nonumber
	P^{\textrm{Gibbs}}_W(w) \triangleq \frac{\exp\left({-\frac{\lambda f(w)}{\sigma_1}}\right)Q(w)}{\int_w \exp\left({-\frac{\lambda f(w)}{\sigma_1}}\right)Q(w)\mathrm{d}w}.
\end{equation}
We prove the following theorem in the Appendix:
\begin{theorem}\label{Max multiscale shannon entropy}
	The solutions to the maximization problems (\ref{Lagrangian Shannon}) and \eqref{Lagrangian differential} are unique and are outputs of Algorithm \ref{RC}.
\end{theorem}
Notice that Algorithm \ref{RC} consists of three phases: (I) The initialization with a Gibbs distribution at line 1. (II) A `renormalization group' phase at lines 2--4 in which the degrees of freedom are eliminated one by one to obtain the intermediate distributions $U^{(i)}_{W^{(i)}}$ for all scales $2\leq i\leq d$ in an increasing (coarsening) order. (III) A refinement phase at line 5, in which the desired distribution $P_W^{\star}$ is obtained by concatenating the intermediate distributions along the decreasing (refining) order by conditional distributions. 
As we shall see in the next subsection, Algorithms \ref{Bayesian renormalization} and \ref{MT} also have a similar structure, though the renormalization step will be replaced by \emph{Bayesian renormalization}. This, in turn, will introduce a Bayesian variant of the renormalization group.

\subsection{Multiscale relative entropy minimization}\label{minimum multiscale entropy problem}
Let $f$ be an arbitrary measurable function called the energy. For any $\lambda>0$, and a fixed prior distribution $Q_W$, as mentioned in Section \ref{Introduction section}, if 
\begin{equation}
	\nonumber
    \widehat{P}_W(w)\triangleq \argmin_{P_W}\left\{\E[f(W)]+\lambda D(P_W\|Q_W)\right\},
\end{equation}
where $W\sim P_W$, then
$
    \widehat{P}_W(w)= \frac{\exp\left(-\frac{f(w)}{\lambda}\right)Q(w)}{\int_w \exp\left(-\frac{f(w)}{\lambda}\right)Q(w)\mathrm{d}w} 
$
is the Gibbs--Boltzmann distribution.
Now, consider the following multiscale generalization of the previous problem:
\begin{align}\label{multiscale Gibbs minimizing}
    P^{\star}_W\triangleq \argmin_{P_W}\left\{\E[f(W)]+ \lambda D_{(\sigma, \mathbf{T})}(P_W\|Q_W)\right\}. 
\end{align}
Notice that this is the Lagrangian of the problem of minimizing multiscale relative entropy $D_{(\sigma, \mathbf{T})}(P_W\|Q_W)$ under the mean constraint $\E[f(W)]=\mu$. It was shown in \cite{asadi2020chainingchainrule} that (\ref{multiscale Gibbs minimizing}), in the special case of decimation transformation, also has a unique minimizer which can be characterized with the proposed Marginalize-Tilt (MT) algorithm, restated here as Algorithm \ref{MT}. In this paper, we show that for arbitrary scale transformations, the solution to (\ref{multiscale Gibbs minimizing}) is unique and given with Algorithm \ref{Bayesian renormalization}---a more general version of the MT Algorithm.  
First, the definition of  \emph{tilted distribution} is required which is basically the geometric mean between two distributions:
\begin{definition}[Tilted distribution\footnote{This is also known as \emph{generalized escort distribution} in statistical physics literature; see e.g. \cite{bercher2012simple}.}]
Let $\theta\in [0,1]$.
	If $P$ and $Q$ are distributions on a discrete set $\mathcal{A}$, then for all $a\in \mathcal{A}$, we define the tilted distribution $(P,Q)_{\theta}$ with
\begin{equation}
	\nonumber
	(P,Q)_{\theta}(a)= \frac{P^{\theta}(a)Q^{1-\theta}(a)}{\sum_{x\in \mathcal{A}}P^{\theta}(x)Q^{1-\theta}(x)}.
\end{equation}
For continuous random variables it is defined analogously except by replacing the sum in the denominator with an integral.
\end{definition}
Define the Gibbs distribution at the finest (microscopic) scale as  
\begin{equation}
	P^{\textrm{Gibbs}}_W(w) \triangleq \frac{\exp\left({-\frac{f(w)}{\lambda\sigma_1}}\right)Q(w)}{\int_w \exp\left({-\frac{f(w)}{\lambda\sigma_1}}\right)Q(w)\textrm{d}w},\label{microscopic Gibbs dist}
\end{equation}
which would be the minimizer of (\ref{multiscale Gibbs minimizing}) had we had $\sigma_2 = \dots = \sigma_{d}=0$. Algorithm \ref{Bayesian renormalization} receives the microscopic Gibbs distribution, the prior distribution $Q_W$, and the values of $\sigma_2, \dots, \sigma_{d}$ as its input and outputs the desired multiscale Gibbs distribution $P^{\star}_W$---the minimizer of (\ref{multiscale Gibbs minimizing}).
\begin{algorithm}
    \caption{}
    \label{Bayesian renormalization}
    \begin{algorithmic}[1]      
                 \State $U^{(1)}_{W^{(1)}}\gets P^{\mathrm{Gibbs}}_{W}$ \Comment{Initial microscopic Gibbs distribution}\label{initialization of MT}
                 \For{$i=2 \texttt{ to } d$}
        	 	 \State $M^{(i)}_{W^{(i)}}\gets T_{i-1}\left(U^{(i-1)}_{W^{(i-1)}}\right)$ \Comment{Coarse-graining }
        	 	 \State $U^{(i)}_{W^{(i)}}\gets \left(M^{(i)}_{W^{(i)}},Q_{W^{(i)}}\right)_{\frac{\sigma_1+\dots+\sigma_{i-1}}{\sigma_{1}+\dots +\sigma_i}}$ \Comment{Bayesian renormalization}
     			 \EndFor
            \State \textbf{return} $P^{\star}_{W}= U^{(d)}_{W^{(d)}}U^{(d-1)}_{W^{(d-1)}|W^{(d)}}\dots U^{(1)}_{W^{(1)}|W^{(2)}}$ \Comment{Refinement}
    \end{algorithmic}
\end{algorithm}
\begin{algorithm}
    \caption{Marginalize-Tilt (MT) \cite{asadi2020chainingchainrule}}
    \label{MT}
    \begin{algorithmic}[1]      
                 \State $U^{(1)}_{W_1\dots W_d}\gets P^{\mathrm{Gibbs}}_{W}$ \Comment{Initial microscopic Gibbs distribution}\label{initialization of MT}
                 \For{$i=2 \texttt{ to } d$} 
        	 	 \State $M^{(i)}_{W_1\dots W_{d-i+1}}\gets U^{(i-1)}_{W_1\dots W_{d-i+1}}$ \Comment{Marginalization}
        	 	 \State $U^{(i)}_{W_1\dots W_{d-i+1}}\gets \left(M^{(i)}_{W_1\dots W_{d-i+1}},Q_{W_1\dots W_{d-i+1}}\right)_{\frac{\sigma_{1}+\dots+\sigma_{i-1}}{\sigma_{1}+\dots +\sigma_i}}$ \Comment{Tilting}
     			 \EndFor
            \State \textbf{return} $P^{\star}_{W}= U^{(d)}_{W_1}U^{(d-1)}_{W_2|W_1}\dots U^{(1)}_{W_d|W_1\dots W_{d-1}}$ \Comment{Refinement}
    \end{algorithmic}
\end{algorithm}

\begin{theorem}\label{multiscale relative entropy RG theorem}
	The solution to the maximization problem (\ref{multiscale Gibbs minimizing}) is  unique and is the output of Algorithm \ref{Bayesian renormalization}. For the special case of decimation transformation, Algorithm \ref{Bayesian renormalization} reduces to Algorithm \ref{MT}.
\end{theorem} 
For a proof, see the Appendix. As per \cite{asadi2020chainingchainrule}, we call $(\lambda\sigma_1,\dots,\lambda\sigma_d)$ the temperature vector of $P^{\star}_W$.

\subsection{Multi-objective optimization viewpoint}\label{multiobjective section}
Notice that when maximizing multiscale entropy under a mean constraint, for different values of length scale coefficients $\sigma= (\sigma_1,\dots, \sigma_d)$ we are finding the set of Pareto optimal points of the multi-objective optimization with the entropies at different scales as the objective functions. Therefore maximum multiscale entropy can also be interpreted as a linear scalarization of a multi-objective optimization problem (see e.g. \cite{hwang2012multiple} for a definition of linear scalarization). Thus, roughly speaking, maximum multiscale entropy distributions maximize entropies at multiple scales simultaneously.  
\section{Maximum multiscale entropy and multivariate Gaussians}\label{Gaussians section}
Here, we show that the MT algorithm is closed on the family of multivariate Gaussian distributions. We also show that the same fact holds for Algorithm \ref{RC} in the special case of decimation transformation.

\begin{theorem}\label{Gaussian closedness}
    Assume that the microscopic Gibbs distribution $P^{\textrm{Gibbs}}_W$ is multivariate Gaussian. Then for decimation transformation, the output of Algorithm \ref{RC} is multivariate Gaussian as well. Furthermore, if the prior $Q_W$ is also multivariate Gaussian, then so is the output of the MT algorithm. In these cases, these algorithms simplify to parameter computations of multivariate Gaussians.
\end{theorem}
\noindent
For a precise proof, see the Appendix. A proof sketch is as follows: Based on a well-known property of multivariate Gaussians, marginalizing out some of its random variables keeps the distribution as multivariate Gaussian. Also, scaling a Gaussian or tilting it towards another Gaussian keeps the resultant distribution as Gaussian. 
Therefore, the renormalization group phase of Algorithms \ref{RC} and \ref{MT} keep all the distributions as Gaussians. Hence, all the intermediate distributions $U^{(i)}_{W_1\dots W_{d-i+1}}$ are multivariate Gaussians. The proof is complete by repeatedly applying the following proposition in the refinement phase, which states that concatenating two Gaussians with conditional distribution results in another Gaussian distribution: 
\begin{proposition}[Gaussian concatenation]\label{Gaussian concatenation}
Let $U_{W_1}^{(1)}$ and $U_{W_1W_2}^{(2)}$ be multivariate Gaussian distributions. Then $P_{W_1W_2}\stackrel{\Delta}{=}U_{W_1}^{(1)}U_{W_2|W_1}^{(2)}$ is multivariate Gaussian as well.
  \end{proposition}
  \noindent
Proposition \ref{Gaussian concatenation} may not be new, however, we were not able to find it in the literature; for a precise form and proof, see the Appendix.
Note that when the energy function is a definite quadratic function $f(W)=W^TKW$, where $K\succ 0$ is a positive definite matrix, and the prior $Q_W$ is multivariate Gaussian, then based on its definition in Subsection \ref{minimum multiscale entropy problem}, the microscopic Gibbs distribution is a multivariate Gaussian distribution. Hence, based on the previous argument, the multiscale Gibbs distribution $P^{\star}_W$ is multivariate Gaussian as well.

\section{Multiscale entropic regularization of neural networks}\label{neural network section}

Let $\phi$ denote the hyperbolic tangent activation function.   
	 Consider a $d$ layer feedforward (residual) neural network $h_W:\mathcal{X}\to \mathcal{Y}$ with parameters
	$
		W\triangleq ({W}_1,{W}_2,\dots,{W}_d)\in \mathcal{W}=\mathcal{W}_1\times\dots \times \mathcal{W}_d,
	$ 
	where for all $1\leq k\leq d$, ${W}_k\in\mathbb{R}^{m\times m}$, 
	and given input $x$, the relations between the hidden layers $h_0,h_1,\dots,h_d$ and the output layer $h$ are as follows: $h_0\triangleq x, h_{i}\triangleq \sigma\left({W}_ih_{i-1}\right)+h_{i-1}, \textrm{~for all ~} 1\leq i \leq d, h_W(x)\triangleq h_d.$
	Let $\ell$ be the squared loss, that is, for the network with parameters $W$ and for any example $z=(x,y)\in \mathsf{Z}$, we have
	$
		\ell(W,z)\triangleq |h_W(x)-y|_2^2
	$.
	The following assumption is adopted from \cite{asadi2020chainingchainrule}, named as multilevel regularization: $\mathcal{W}_k\triangleq \left\{{W}\in \mathbb{R}^{m\times m}: \|{W}\|_2\leq \frac{1}{d} \right\}$ for all $1\leq k \leq d$, which is similarly used in \cite{hardt2016identity}. Let $S=(z_1,\dots,z_n)\sim \mu^{\otimes n}$ denote the training set in supervised learning. For any $w\in\mathcal{W}$, let 
$
L_{\mu}(w)\triangleq \E[\ell(w,Z)] 
$
denote the statistical (or population) risk of hypothesis $h_w$, where $~ Z\sim \mu$. For a given training set $S$, the empirical risk of hypothesis $h_w$ is defined as 
$
L_{S}(w)\triangleq \frac{1}{n}\sum_{i=1}^n \ell(w,Z_i).
$
The following lemma controls the difference between consecutive hidden layers of the neural network:
\begin{lemma}\label{chaining link distance} 
	For any $2\leq i\leq d$ and all $x\in \mathcal{X}$,
	\begin{equation}
		|h_i(x)-h_{i-1}(x)|_2 \leq \frac{e|x|_2}{d}.\nonumber
	\end{equation}
\end{lemma}
\begin{proof} 
Since $h_i=\sigma({W}_ih_{i-1})+h_{i-1}$, based on induction on $i$ and the triangle inequality, we have $|h_{i-1}|_2\leq \exp\left(\frac{i-1}{d}\right)|x|_2\leq e|x|_2.$ Therefore
\begin{align}
    \nonumber
    |h_{i}(x)-h_{i-1}(x)|_2 &= |\sigma(W_ih_{i-1})|_2\\
                              &\leq |W_ih_{i-1}|_2 \nonumber\\
                              &\leq \|W_i\|_2 |h_{i-1}|_2 \nonumber\\
                              &\leq \frac{e|x|_2}{d}.\nonumber
\end{align}
\end{proof}
We assume that the instances have bounded norm, namely, $\mathcal{X}=\{x\in \mathbb{R}^m: |x|_2\leq R\}$.
Based on Lemma \ref{chaining link distance} and a similar technique to \cite{asadi2020chainingchainrule}, we can obtain the following multiscale entropic generalization bound:
\begin{theorem}\label{CMI generalization deep nets theorem} 
	Let $P_S\to P_{W|S}\to P_W$. We have the following generalization bound, where $C$ is a constant, $\gamma\triangleq (\gamma_1,\dots,\gamma_d)$ a vector of positive reals, and $Q_W$ a prior distribution:
    \begin{align}
        \E\left[ L_{\mu}(W)\right]  \leq \E\left[ L_S(W)\right]+ \frac{C}{d\sqrt{n}}\inf_{\gamma, Q_W}\sum_{i=1}^{d}\left(\gamma_i D\left(P_{W_1\dots W_{d-i+1}|S}\middle\|Q_{W_1\dots W_{d-i+1}}\middle|P_S\right)+\frac{1}{4\gamma_i} \right). \nonumber
    \end{align}    
\end{theorem}
\noindent 
See the Appendix for a proof. For fixed $\gamma=(\gamma_1,\dots, \gamma_d)$ and $Q_W$, and any $z^n\in \mathsf{Z}^n$, let
\begin{equation}\label{multiscale relative entropy minimization}
	P^{\star}_{W|S=z^n}\triangleq \argmin_{P_{W}} \left\{ \E[L_{z^n}(W)]+ \sum_{i=1}^d \sigma_i D\left(P_{W_1\dots W_{d-i+1}}\middle\|Q_{W_1\dots W_{d-i+1}}\right)\right\},
\end{equation}
where $\sigma_i\triangleq \frac{C\gamma_i}{d\sqrt{n}}$ for all $1\leq i \leq d$. Note that (\ref{multiscale relative entropy minimization}) has the same form as (\ref{multiscale Gibbs minimizing}) with the decimation transformation with $\lambda=1$, therefore we can use the MT algorithm to obtain $P^{\star}_{W|S=z^n}$ for any $z^n\in\mathsf{Z}^n$. To obtain \emph{excess risk bounds} from the generalization bound of Theorem \ref{CMI generalization deep nets theorem}, we employ a technique from \cite{xu2017information}: Since $P^{\star}_{W|S=z^n}$ minimizes the expression in (\ref{multiscale relative entropy minimization}), one can obtain excess risk bounds by plugging in a fixed distribution $\widehat{Q}_W$ concentrated around a population risk minimizer $\widehat{w}=(\widehat{w}_1,\dots,\widehat{w}_d)$ and independent from $S$.  
We now can state the main result of this section.
\begin{theorem}\label{DPG thoerem}
Define the data processing gain at scale $i$ by 
\begin{equation}
	\nonumber
	\mathrm{DPG}(i)\triangleq \sqrt{D(\widehat{Q}_{W}\|{Q}_{W})}-\sqrt{D(\widehat{Q}_{W^{(i)}}\|{Q}_{W^{(i)}})}.
\end{equation}
Then the difference between the excess risk bounds of the single-scale Gibbs posterior and the multiscale Gibbs posterior, when each are optimized over their hyper-parameter (temperature) values, is equal to $\frac{C}{d\sqrt{n}}\sum_{i=1}^d \mathrm{DPG}(i)$ and is positive.
\end{theorem}
Hence we can always guarantee a tighter excess risk for the multiscale Gibbs posterior than for the single-scale Gibbs posterior. For example, if the weights of the network take discrete values, then we can take $\widehat{Q}_W$ to be the Dirac delta measure on $\widehat{w}=(\widehat{w}_1,\dots,\widehat{w}_d)$. In this case, for any prior distribution $Q_W$, there exists $\gamma=(\gamma_1,\dots,\gamma_d)$ such that 
\begin{equation}
	\nonumber
	\E\left[ L_{\mu}(W)\right] - \inf_{w\in\W}L_{\mu}(w)\leq \frac{C}{d\sqrt{n}}\sum_{i=1}^d \sqrt{\log \frac{1}{Q_{W_1\dots W_i}(\widehat{w}_1, \dots ,\widehat{w}_i)}}.
\end{equation}
However, the excess risk bound for the single-scale Gibbs distribution when optimized over its temperature parameter is
\begin{equation}
	\nonumber
	\E\left[ L_{\mu}(W)\right] - \inf_{w\in\W}L_{\mu}(w)\leq\frac{C}{\sqrt{n}}\sqrt{\log \frac{1}{Q_{W_1\dots W_d}(\widehat{w}_1, \dots ,\widehat{w}_d)}}.
\end{equation}
The difference between the right sides of these bounds is given by Theorem \ref{DPG thoerem}.
For a precise proof of Theorem \ref{DPG thoerem} and an example when the synaptic weights take continuous values, see the Appendix.
  
\subsection{Teacher-Student example}\label{TeacherStudent section}
A teacher-student scenario, first studied in \cite{saad1995line}, has the advantage of facilitating the evaluation of $\mathrm{DPG}(i)$.
Let data be generated from a teacher residual network with depth ${d}/{M}$, where $M>1$. This is equivalent to a depth $d$ teacher network with identity mappings at the first $d(1-\frac{1}{M})$ layers. Hence $\inf_{w\in \mathcal{W}}L_{\mu}(w)=0$, and we choose $\widehat{w}$ as the weights of the teacher network. Assume an i.i.d. Gaussian prior $Q_{W_1\dots W_d}$ centered at zero.
Hence 
 $
     q_1=Q_{W_1}(\widehat{w}_1)\approx\dots \approx Q_{W_{d(1-\frac{1}{M})}}(\widehat{w}_{d(1-\frac{1}{M})})
 $
 and assume
 $
     q_2=Q_{W_{d(1-\frac{1}{M})+1}}(\widehat{w}_{d(1-\frac{1}{M})+1})\approx\dots \approx Q_{W_d}(\widehat{w_d}),
 $ 
 where $q_1\gg q_2$. We show in the Appendix that $\sum_{i=1}^d \mathrm{DPG}(i)\approx (\log \frac{1}{q_2})^{\frac{1}{2}}d^{\frac{3}{2}}(\frac{M-\frac{2}{3}}{M^{\frac{3}{2}}}),$ which quantifies the improvement gap. 

\subsection{Experiment}\label{Experiment section}
\begin{figure}
	\begin{center}
    \includegraphics[width=0.7\textwidth]{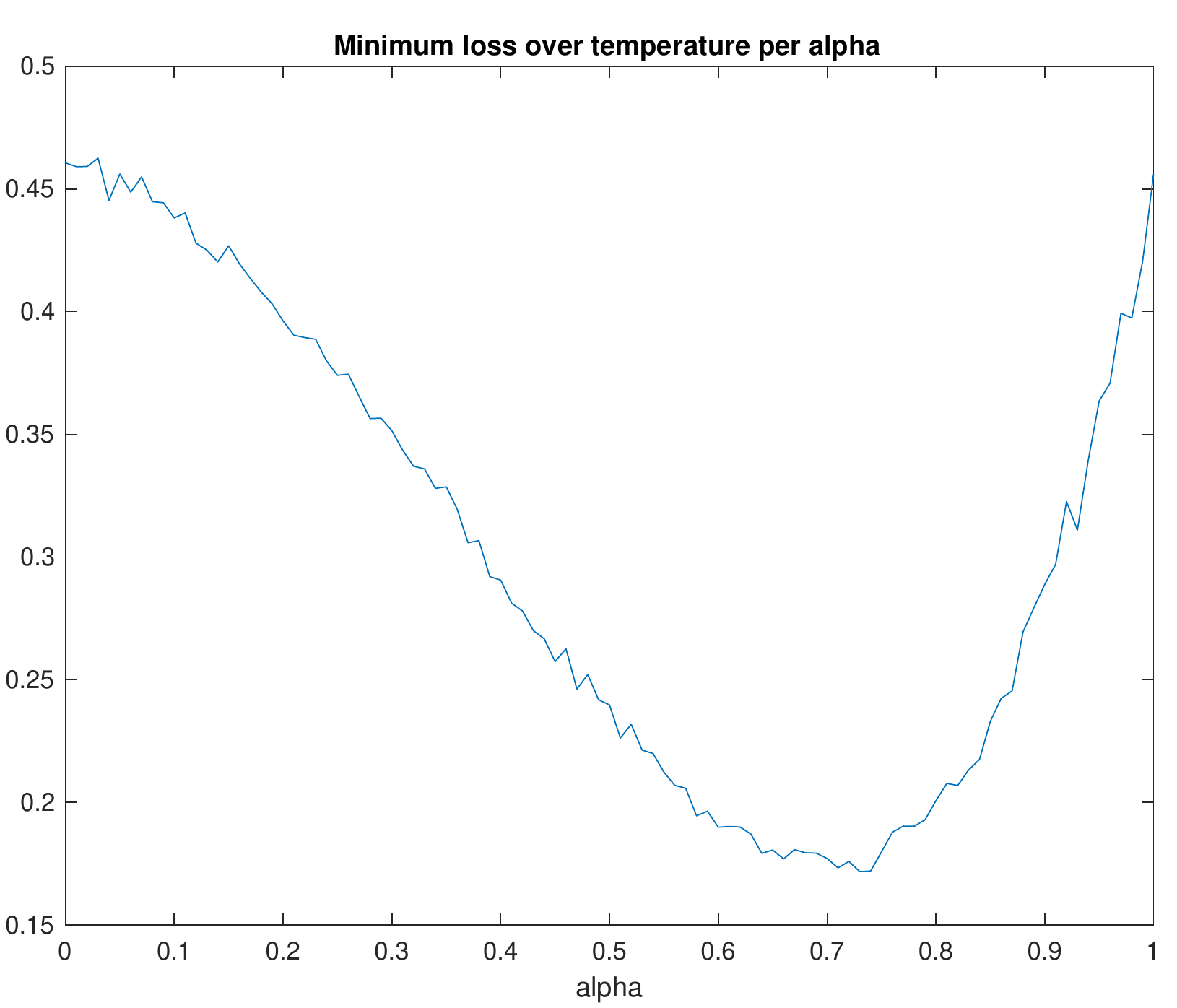}
  \end{center}
  \caption{
  Minimum loss over temperature per $\alpha$. The vertical axis denotes the population risk of a hypothesis randomly chosen by the multiscale Gibbs posterior. Note that $\alpha=0$ corresponds to the single-scale Gibbs distribution, and $\alpha\to 1$ corresponds to learning with random features.}
  \label{min loss per alpha}
\end{figure}
Assume that the temperature vector of the multiscale Gibbs posterior $\sigma=(\sigma_1,\dots,\sigma_d)$ is such that $\sigma_1$ takes arbitrary positive values and the rest of the parameters are determined inductively with the following equations:
$
    \frac{\sigma_i}{\sigma_1+\dots+\sigma_i}=\alpha\in[0,1)$ for all $2\leq i\leq d.
$
Hence, the tilting indices in the MT algorithm are all equal to $\alpha$ and we can represent the temperature vector $\lambda\in \mathbb{R}^d$ with just two positive parameters $(\alpha, \sigma_1)$. Notice that when $\alpha=0$, then the multiscale Gibbs distribution is simply equivalent to the single-scale Gibbs distribution. Moreover, the case $\alpha\to 1$ corresponds to sampling the first $d-1$ layers randomly from the prior distribution, and only training the last layer, a condition similar to random feature learning \cite{rahimi2008random}. In the following experiment, assume that we have a teacher network and a student network. For different values of $\alpha\in[0,1)$, we minimize the performance of the algorithm over different values of the temperature $\sigma_1$. We use the Gauss--Newton matrix at the origin to obtain Gaussian approximations to the microscopic Gibbs distribution, then use Theorem \ref{Gaussian closedness}.  See Figure \ref{min loss per alpha}. Notice that there exist intermediate values for $\alpha$ for which the population risk is much better than extreme cases of $\alpha=0$ and $\alpha\to 1$. For more details, see the Appendix. 

\section*{Acknowledgement}
Amir Asadi thanks Siddhartha Sarkar for useful discussions on renormalization group theory.

\appendix

\section{Proofs for Section \ref{Maximum multiscale entropy section}}
Here, we present the proofs of maximum multiscale entropy results.
\subsection{Multiscale Shannon and differential entropy maximization}
For the proof of Theorem \ref{Max multiscale shannon entropy} we first require the following lemmas. The first lemma is used for proving the optimality of the Gibbs distribution for maximizing Shannon entropy:
\begin{lemma}\label{Gibbs is the maximizer entropy}
  Let $f: {\cal A}\rightarrow \mathbb{R}$ be such that $\sum_{w\in {\cal A}}\exp(-f(w))<\infty$, where ${\cal A}$ is a finite or countably infinite set. Then for any $P_{W}$ defined on ${\cal A}$,
  \begin{equation}
      \nonumber
      H(W)-\lambda\mathbb{E}[f(W)]= -D\left(P_{W}\middle\|{P}_{W}^{\rm{Gibbs}}\right)+\log \left(\sum_{w\in {\cal A}}\exp(-\lambda f(w))\right),
  \end{equation}
  where 
  $${P}_{W}^{\rm{Gibbs}}(w)\triangleq\frac{\exp(-\lambda f(w))}{\sum_{w\in {\cal A}}\exp(-\lambda f(w))}, \quad  w\in {\cal A},$$ is the Gibbs--Boltzmann distribution.
\end{lemma}
\begin{proof}
\begin{align}
    H(W)-\lambda\E[f(W)]&=-\sum_{w\in \mathcal{A}}P(w)\log P(w) -\lambda \sum_{w\in \mathcal{A}}f(w)P(w)\nonumber\\
                        & = -\sum_{w\in \mathcal{A}}P(w)\log \frac{P(w)}{\exp(-\lambda f(w))}\\
                        & = -\sum_{w\in \mathcal{A}}P(w)\log \frac{P(w)}{\frac{\exp(-\lambda f(w))}{\sum_{w\in {\cal A}}\exp(-\lambda f(w))}}+ \log \left(\sum_{w\in {\cal A}}\exp(-\lambda f(w))\right)\\
                    &=-D\left(P_{W}\middle\|{P}_{W}^{\rm{Gibbs}}\right)+ \log \left(\sum_{w\in {\cal A}}\exp(-\lambda f(w))\right).
\end{align}
    
\end{proof}
As a corollary of Lemma \ref{Gibbs is the maximizer entropy}, the maximizer of $H(W)-\lambda\mathbb{E}[f(W)]$ is given by the Gibbs distribution ${P}_{W}^{\rm{Gibbs}}$. The counterpart of Lemma \ref{Gibbs is the maximizer entropy} for continuous random variables and differential entropy is as follows:
\begin{lemma}\label{Gibbs is the maximizer diff entropy}
  Let $f: {\cal A}\rightarrow \mathbb{R}$ be such that $\int_{w\in {\cal A}}\exp(-f(w))\mathrm{d}w<\infty$, where ${\cal A}$ is an uncountable set. Then for any $P_{W}$ defined on ${\cal A}$, 
  \begin{equation}
      \nonumber
      h(W)-\lambda\mathbb{E}[f(W)]= -D\left(P_{W}\middle\|{P}_{W}^{\rm{Gibbs}}\right)+\log \left(\int_{w\in {\cal A}}\exp(-\lambda f(w))\mathrm{d}w\right),
  \end{equation}
  where 
  $${P}_{W}^{\rm{Gibbs}}(w)\triangleq\frac{\exp(-\lambda f(w))}{\int_{w\in {\cal A}}\exp(-\lambda f(w))\mathrm{d}w}, \quad w\in {\cal A},$$ is the Gibbs--Boltzmann distribution.
\end{lemma}
\begin{proof}
\begin{align}
    h(W)-\lambda\E[f(W)]&=-\int_{w\in \mathcal{A}}P(w)\log P(w)\dw -\lambda \int_{w\in \mathcal{A}}f(w)P(w)\dw\nonumber\\
                        & = -\int_{w\in \mathcal{A}}P(w)\log \frac{P(w)}{\exp(-\lambda f(w))}\dw\\
                        & = -\int_{w\in \mathcal{A}}P(w)\log \frac{P(w)}{\frac{\exp(-\lambda f(w))}{\int_{w\in {\cal A}}\exp(-\lambda f(w))}}\dw+ \log \left(\int_{w\in {\cal A}}\exp(-\lambda f(w))\right)\\
                    &=-D\left(P_{W}\middle\|{P}_{W}^{\rm{Gibbs}}\right)+ \log \left(\int_{w\in {\cal A}}\exp(-\lambda f(w))\right).
\end{align}
    
\end{proof}
Let $H_{\alpha}(P)$ denote the R\'{e}nyi entropy of order $\alpha$ of discrete distribution $P$, which for $\alpha\in (0,1)\cup (1,\infty)$ is defined as 
$$H_{\alpha}(P)\triangleq \frac{1}{1-\alpha}\log \sum_{w\in {\cal A}} P^{\alpha}(w).$$
Similarly, let $h_{\alpha}(P)$ denote the R\'{e}nyi differential entropy of order $\alpha$ of continuous distribution $P$, which for $\alpha\in (0,1)\cup (1,\infty)$ is defined as 
$$H_{\alpha}(P)\triangleq \frac{1}{1-\alpha}\log \int_{w\in {\cal A}} P^{\alpha}(w)\mathrm{d}w.$$
The following two lemmas show how to linearly combine an entropy with a relative entropy, using scaled distributions:
\begin{lemma}\label{Shannon entropy KL sum}
	Let $P$ and $Q$ be two discrete distributions and $\theta \geq 0$. We have 
	\begin{equation}\nonumber
	H(P)-\theta D(P\|Q) = H_{\frac{\theta}{1+\theta}}(Q)-(1+\theta)D\left(P\middle\|(Q)_{\frac{\theta}{1+\theta}}\right).
	\end{equation}
\end{lemma}

\begin{proof}
  \begin{align*}
  H(P)-\theta D(P\|Q)
    &=-\sum P(w)\log P(w)  -\theta\sum P(w)\log\frac{P(w)}{Q(w)}\\
  &=-\sum P(w)\log\frac{P(w)^{1+\theta}}{Q(w)^{\theta}} \\
  &=-(1+\theta)\sum P(w)\log\frac{P(w)}{Q(w)^{\frac{\theta}{1+\theta}}} \\
  &=H_{\frac{\theta}{1+\theta}}(Q)-(1+\theta)D\left(P\middle\|(Q)_{\frac{\theta}{1+\theta}}\right).
    \end{align*}
\end{proof}

\begin{lemma}\label{differential entropy proposition}
  Let $P$ and $Q$ be two continuous distributions and $\theta \geq 0$. Then
  \begin{align*}
  h(P)-\theta D(P\|Q) =h_{\frac{\theta}{1+\theta}}(Q)-(1+\theta)D\left(P\middle\|(Q)_{\frac{\theta}{1+\theta}}\right).
\end{align*}
\end{lemma}

\begin{proof}
  \begin{align*}
  h(P)-\theta D(P\|Q)
    &=-\int P(w)\log P(w) \mathrm{d}w -\theta\int P(w)\log\frac{P(w)}{Q(w)}\mathrm{d}w\\
  &=-\int P(w)\log\frac{P(w)^{1+\theta}}{Q(w)^{\theta}} \mathrm{d}w\\
  &=-(1+\theta)\int P(w)\log\frac{P(w)}{Q(w)^{\frac{\theta}{1+\theta}}}\mathrm{d}w \\
  &=h_{\frac{\theta}{1+\theta}}(Q)-(1+\theta)D\left(P\middle\|(Q)_{\frac{\theta}{1+\theta}}\right).
    \end{align*}
\end{proof}

For simplicity of the proofs, we assume that all alphabets are standard Borel spaces, which guarantees the existence of regular conditional probabilities and reverse random transformations. Therefore, as a corollary of the chain rule of relative entropy, we have the following:
\begin{lemma}\label{chain rule based}
Let $P_{W_1}\to T_{W_2|W_1} \to P_{W_2}$ and $Q_{W_1} \to T_{W_2|W_1} \to Q_{W_2}$. Then
    \begin{equation}
        \nonumber
        D(P_{W_1}\|Q_{W_1}) = D(P_{W_2}\|Q_{W_2}) + D(P_{W_1|W_2}\|Q_{W_1|W_2}|P_{W_2}).
    \end{equation}
\end{lemma}
\begin{proof}
Expanding $D(P_{W_1W_2}\|Q_{W_1W_2})$ in two different ways based on the chain rule of relative entropy gives
\begin{align}
    D(P_{W_1W_2}\|Q_{W_1W_2}) &= D(P_{W_2}\|Q_{W_2}) + D(P_{W_1|W_2}\|Q_{W_1|W_2}|P_{W_2})\\
                              &= D(P_{W_1}\|Q_{W_1}) + D(P_{W_2|W_1}\|Q_{W_2|W_1}|P_{W_1}).
\end{align}
The conclusion follows from noting that $D(P_{W_2|W_1}\|Q_{W_2|W_1}|P_{W_1})=0$.
\end{proof}
As a corollary of Lemma \ref{chain rule based} for deterministic random transformation $T_{W_2|W_1}$, we have the following:
\begin{lemma}\label{Chain rule T} For any function $T$, we have
    \begin{equation}
    \nonumber
        D(P_{W}\|Q_W) = D(P_{T(W)}\|Q_{T(W)})+D(P_{W|T(W)}\|Q_{W|T(W)}|P_{T(W)}).
    \end{equation}
\end{lemma}
We finally arrive at the proof of Theorem \ref{Max multiscale shannon entropy}:
\begin{proof}[\bf Proof of Theorem \ref{Max multiscale shannon entropy}]
Based on Lemma \ref{Gibbs is the maximizer entropy}, we can write 
\begin{align}
    H_{(\mathbf{\sigma},\mathbf{T})}(W) -\lambda\mathbb{E}[f(W)] &= \sum_{i=1}^d \sigma_i H\left(W^{(i)}\right)-\lambda\mathbb{E}[f(W)]\nonumber\\
    &=\sum_{i=2}^{d} \sigma_i H\left(W^{(i)}\right)+ \left(\sigma_1 H\left(W\right)-\lambda\mathbb{E}[f(W)]\right)\nonumber\\
    &=\sum_{i=2}^{d} \sigma_i H\left(W^{(i)}\right)+ \sigma_1\left(H\left(W\right)-\frac{\lambda}{\sigma_1}\mathbb{E}[f(W)]\right)\nonumber\\
    &= \underbrace{\sum_{i=2}^{d} \sigma_i H\left(W^{(i)}\right)-\sigma_1D\left(P_{W}\middle\|{P}_{W}^{\rm{Gibbs}}\right)}_{A}\nonumber\\
    &\quad +\sigma_1 \log \left(\sum_{w\in {\cal A}}\exp(-\lambda f(w))\right).\nonumber
\end{align}
Note that $\sum_{w\in {\cal A}}\exp(-\lambda f(w))$ does not depend on $P_W$, therefore it suffices to find the maximizer of $A$. For all $1\leq i \leq d-1$, let $P_{W^{(i+1)}}=T_i\left(P_{W^{(i)}}\right)$ and $M^{(i+1)}_{W^{(i+1)}}=U^{(i)}_{W^{(i+1)}}=T_i\left(U^{(i)}_{W^{(i)}}\right)$ be the image measures of function $T_i$.  
Further, assume that \emph{reverse random transformations} $P_{W^{(i)}|W^{(i+1)}}$ and $U^{(i)}_{W^{(i)}|W^{(i+1)}}$ exist\footnote{By assuming regular conditional probabilities as discussed above.} such that  $P_{W^{(i+1)}}\to P_{W^{(i)}|W^{(i+1)}} \to P_{W^{(i)}}$ and $U^{(i)}_{W^{(i+1)}}\to U^{(i)}_{W^{(i)}|W^{(i+1)}} \to U^{(i)}_{W^{(i)}}$. 
We can rewrite $A$ as: 
\begin{align}
	A&=\sum_{i=2}^{d} \sigma_iH\left(W^{(i)}\right)-\sigma_1 D\left(P_{W}\middle\|P^{\rm{Gibbs}}_W\right)\nonumber\\
	&=\sum_{i=2}^{d} \sigma_iH\left(W^{(i)}\right)-\sigma_1 D\left(P_{W^{(1)}}\middle\|U^{(1)}_{W^{(1)}}\right)\nonumber\\
	&=\sum_{i=2}^{d} \sigma_iH\left(W^{(i)}\right)-\sigma_1 \left(D\left(P_{W^{(2)}}\middle\|M^{(2)}_{W^{(2)}}\right)+D\left(P_{W^{(1)}|W^{(2)}}\middle\|U^{(1)}_{W^{(1)}|W^{(2)}}\middle|P_{W^{(2)}}\right)\right)\label{based on chain rule}\\
	&=\sum_{i=3}^{d} \sigma_iH\left(W^{(i)}\right)+\sigma_2\left(H\left(W^{(2)}\right)-\frac{\sigma_1}{\sigma_2} D\left(P_{W^{(2)}}\middle\|M^{(2)}_{W^{(2)}}\right)\right)\nonumber\\
	&\quad-\sigma_1D\left(P_{W^{(1)}|W^{(2)}}\middle\|U^{(1)}_{W^{(1)}|W^{(2)}}\middle|P_{W^{(2)}}\right)\nonumber\\
	&=\sum_{i=3}^{d} \sigma_iH\left(W^{(i)}\right)+\sigma_2\left(H_{\frac{\sigma_1}{\sigma_1+\sigma_2}}\left(U^{(2)}_{W^{(2)}}\right)-\left(1+\frac{\sigma_1}{\sigma_2}\right)D\left(P_{W^{(2)}}\middle\|\left(M^{(2)}_{W^{(2)}}\right)_{\frac{\sigma_1}{\sigma_1+\sigma_2}}\right)\right)\nonumber\\
	&\quad-\sigma_1D\left(P_{W^{(1)}|W^{(2)}}\middle\|U^{(1)}_{W^{(1)}|W^{(2)}}\middle|P_{W^{(2)}}\right),\label{mixing eq}\\
	&=\left(\sum_{i=3}^{d} \sigma_iH\left(W^{(i)}\right)+(\sigma_1+\sigma_2)D\left(P_{W^{(2)}}\middle\|\left(M^{(2)}_{W^{(2)}}\right)_{\frac{\sigma_1}{\sigma_1+\sigma_2}}\right)\right)\nonumber\\
	&\quad+\sigma_2\left(H_{\frac{\sigma_1}{\sigma_1+\sigma_2}}\left(U^{(2)}_{W^{(2)}}\right)\right)
	-\sigma_1D\left(P_{W^{(1)}|W^{(2)}}\middle\|U^{(1)}_{W^{(1)}|W^{(2)}}\middle|P_{W^{(2)}}\right)\nonumber\\
	&=\underbrace{\left(\sum_{i=3}^{d} \sigma_iH\left(W^{(i)}\right)+(\sigma_1+\sigma_2)D\left(P_{W^{(2)}}\middle\|U^{(2)}_{W^{(2)}}\right)\right)}_{B}\nonumber\\
	&\quad+\sigma_2\left(H_{\frac{\sigma_1}{\sigma_1+\sigma_2}}\left(U^{(2)}_{W^{(2)}}\right)\right)
	-\sigma_1D\left(P_{W^{(1)}|W^{(2)}}\middle\|U^{(1)}_{W^{(1)}|W^{(2)}}\middle|P_{W^{(2)}}\right),\nonumber
	\end{align}
where (\ref{based on chain rule}) follows from Lemma \ref{Chain rule T} and \eqref{mixing eq} follows from Lemma \ref{Shannon entropy KL sum}. Note that the R\'{e}nyi entropy $H_{\frac{\sigma_1}{\sigma_1+\sigma_2}}\left(U^{(2)}_{W^{(2)}}\right)$ does not depend on $P_W$ and that expression $B$ has the same form as $A$ but with one less entropy. If we repeat the same procedure of using Lemma \ref{Chain rule T} and Lemma \ref{Shannon entropy KL sum} for expression B and so forth, we end up writing $A$ as a sum of R\'{e}nyi entropies which do not depend on $P_W$, minus some conditional relative entropies. But then, the conditional relative entropies can all be set to zero simultaneously by taking
\begin{equation}
    \begin{cases} 
    P^{\star}_{W^{(d)}}&= U^{(d)}_{W^{(d)}}\\
    P^{\star}_{W^{(d-1)}|W^{(d)}}&= U^{(d-1)}_{W^{(d-1)}|W^{(d)}}\\
    &\vdots \\
    P^{\star}_{W^{(1)}|W^{(2)}}&= U^{(1)}_{W^{(1)}|W^{(2)}}
    \end{cases}
\end{equation}
which clearly results in the maximizer of $A$. For any distribution $P_W$, we have
\begin{align}
    P_W&=P_{W^{(1)}W^{(2)}\dots W^{(d)}}\label{functioneq}\\
       &=P_{W^{(d)}}P_{W^{(d-1)}|W^{(d)}}\dots P_{W^{(1)}|W^{(2)}}\label{Markoveq}
\end{align}
where \eqref{functioneq} follows from the fact that $W^{(2)},\dots ,W^{(d)}$ are deterministic funtions of $W=W^{(1)}$, and \eqref{Markoveq} follows from the Markov chain $W^{(1)}\leftrightarrow W^{(2)} \leftrightarrow \dots \leftrightarrow W^{(d)}$. Therefore, we deduce
\begin{equation}
    \nonumber
    P^{\star}_W=U^{(d)}_{W^{(d)}}U^{(d-1)}_{W^{(d-1)}|W^{(d)}}\dots U^{(1)}_{W^{(1)}|W^{(2)}}.
\end{equation}
An analogous reasoning can be used for multiscale differential entropy, instead by using Lemma \ref{Gibbs is the maximizer diff entropy} and Lemma \ref{differential entropy proposition}.
\end{proof}

\subsection{Multiscale relative entropy minimization}
For the proof of Theorem \ref{multiscale relative entropy RG theorem}, we first require the following lemmas:
\begin{lemma}\label{Gibbs relative entropy}
  Let ${\cal A}$ be an arbitrary set and function $f: {\cal A}\rightarrow \mathbb{R}$ be such that \\ $\int_{w\in {\cal A}}\exp\left(-\frac{f(w)}{\lambda}\right)Q_W(w)\mathrm{d}w<\infty$. Then for any $P_{W}$ defined on ${\cal A}$ such that $W\sim P_W$, we have
  \begin{equation}
      \nonumber
      \mathbb{E}[f(W)]+\lambda D(P_W\|Q_W)=\lambda D\left(P_{W}\middle\|{P}_{W}^{\rm{Gibbs}}\right)-\lambda \log \left(\int_{w\in {\cal A}}\exp\left(-\frac{f(w)}{\lambda}\right)Q_W(w)\mathrm{d}w\right) ,
  \end{equation}
  where 
  $${P}_{W}^{\rm{Gibbs}}(w)\triangleq\frac{\exp\left(-\frac{f(w)}{\lambda}\right)Q_{W}(w)}{\int_{w\in {\cal A}}\exp\left(-\frac{f(w)}{\lambda}\right)Q_W(w)\mathrm{d}w}, \quad w\in {\cal A},$$ is the Gibbs--Boltzmann distribution.
\end{lemma}
\begin{proof}
    \begin{align}
        \mathbb{E}[f(W)]+\lambda D(P_W\|Q_W)&= \int_{w\in\mathcal{A}}f(w)P(w)\mathrm{d}w + \lambda \int_{w\in\mathcal{A}} P(w)\log \frac{P(w)}{Q(w)}\dw \nonumber\\
        & = \lambda \int_{w\in\mathcal{A}} P(w)\log \frac{P(w)}{\frac{\exp(-\frac{f(w)}{\lambda})Q(w)}{\int_{w\in\mathcal{A}}\exp(-\frac{f(w)}{\lambda})Q(w)\dw}}\dw \nonumber\\
        & \quad -\lambda \log \left(\int_{w\in {\cal A}}\exp\left(-\frac{f(w)}{\lambda}\right)Q_W(w)\mathrm{d}w\right)\nonumber\\
        &= \lambda D\left(P_{W}\middle\|{P}_{W}^{\rm{Gibbs}}\right)-\lambda \log \left(\int_{w\in {\cal A}}\exp\left(-\frac{f(w)}{\lambda}\right)Q_W(w)\mathrm{d}w\right).\nonumber
    \end{align}
\end{proof}
As a corollary of Lemma \ref{Gibbs relative entropy}, we conclude that the Gibbs--Boltzmann distribution ${P}_{W}^{\rm{Gibbs}}$ is the minimizer of $\mathbb{E}[f(W)]+\lambda D(P_W\|Q_W)$ for $\lambda>0$.

For two distributions $Q$ and $R$, let $D_{\theta}(Q\|R)$ denote the R\'{e}nyi divergence of order $\theta$ between $Q$ and $R$, which for $\theta \in (0,1)\cup (1,\infty)$ is defined as 
$$D_{\theta}(Q\|R)\triangleq \frac{1}{\theta-1}\log \left(\intw Q(w)^{\theta}R(w)^{1-\theta}\right).$$

The following lemma, also appearing in \cite[Theorem 30]{van2014renyi}, shows how to linearly combine relative entropies, using tilted distributions. For the sake of completeness, we provide a proof.
\begin{lemma}\label{Tilted Renyi lemma main}
Let $\theta\in [0,1]$. For any $P, Q$ and $R$,
\begin{align}
	\theta D(P\|Q)+(1-\theta)D(P\|R)= D\left(P\|(Q,R)_{\theta}\right)+(1-\theta)D_{\theta}(Q\|R).\nonumber
\end{align}	
\end{lemma}
\begin{proof}
\begin{align}
    \theta D(P\|Q)+(1-\theta)D(P\|R)&= \theta\intw P(w)\log \frac{P(w)}{Q(w)}\dw +(1-\theta)\intw P(w)\log \frac{P(w)}{R(w)}\dw\nonumber\\
                                    &=\intw P(w)\log \frac{P(w)}{Q(w)^{\theta}R(w)^{(1-\theta)}}\dw \nonumber\\
                                    &=\intw P(w)\log \frac{P(w)}{\frac{Q(w)^{\theta}R(w)^{(1-\theta)}}{\intw Q(w)^{\theta}R(w)^{(1-\theta)}\dw}}\dw \nonumber\\
                                    &\quad -\log \left( \intw Q(w)^{\theta}R(w)^{(1-\theta)}\dw \right)\nonumber\\
                                    &=D\left(P\|(Q,R)_{\theta}\right)+(1-\theta)D_{\theta}(Q\|R).\nonumber
\end{align}
\end{proof}
We can now present the proof of Theorem \ref{multiscale relative entropy RG theorem}:
\begin{proof}[\bf Proof of Theorem \ref{multiscale relative entropy RG theorem}] Based on Lemma \ref{Gibbs is the maximizer entropy}, we can write 
\begin{align}
    \frac{\mathbb{E}[f(W)]}{\lambda}+D_{(\mathbf{\sigma},\mathbf{T})}(W) &= \sum_{i=1}^d \sigma_iD\left(P_{W^{(i)}}\middle\|Q_{W^{(i)}}\right)+\frac{\mathbb{E}[f(W)]}{\lambda}\nonumber\\
    &=\sum_{i=2}^{d} \sigma_iD\left(P_{W^{(i)}}\middle\|Q_{W^{(i)}}\right)+ \frac{1}{\lambda}\left(\lambda\sigma_1 D\left(P_{W^{(i)}}\middle\|Q_{W^{(i)}}\right)+\mathbb{E}[f(W)]\right)\nonumber\\
    &= \underbrace{\sum_{i=2}^{d} \sigma_iD\left(P_{W^{(i)}}\middle\|Q_{W^{(i)}}\right)+\sigma_1D\left(P_{W}\middle\|{P}_{W}^{\rm{Gibbs}}\right)}_{A}\nonumber\\
    &\quad -\sigma_1\log \left(\int_{w\in {\cal A}}\exp\left(-\frac{f(w)}{\lambda\sigma_1}\right)Q_W(w)\mathrm{d}w\right).\nonumber
\end{align}
Note that $\int_{w\in {\cal A}}\exp\left(-\frac{f(w)}{\lambda\sigma_1}\right)Q_W(w)\mathrm{d}w$ does not depend on $P_W$, therefore it suffices to find the minimizer of $A$. 
For all $1\leq i \leq d-1$, let $P_{W^{(i+1)}}=T_i\left(P_{W^{(i)}}\right)$ and $M^{(i+1)}_{W^{(i+1)}}=U^{(i)}_{W^{(i+1)}}=T_i\left(U^{(i)}_{W^{(i)}}\right)$ be the image measures of function $T_i$.   
Further, assume that reverse random transformations $P_{W^{(i)}|W^{(i+1)}}$ and $U^{(i)}_{W^{(i)}|W^{(i+1)}}$ exist\footnote{By assuming regular conditional probabilities as discussed in the previous subsection.} such that  $P_{W^{(i+1)}}\to P_{W^{(i)}|W^{(i+1)}} \to P_{W^{(i)}}$ and $U^{(i)}_{W^{(i+1)}}\to U^{(i)}_{W^{(i)}|W^{(i+1)}} \to U^{(i)}_{W^{(i)}}$. 
We can rewrite $A$ as:
\begin{align}
	A&=\sum_{i=2}^{d} \sigma_iD\left(P_{W^{(i)}}\middle\|Q_{W^{(i)}}\right)+\sigma_1 D\left(P_{W}\middle\|P^{\rm{Gibbs}}_W\right)\nonumber\\
	&=\sum_{i=2}^{d} \sigma_iD\left(P_{W^{(i)}}\middle\|Q_{W^{(i)}}\right)+\sigma_1 D\left(P_{W^{(1)}}\middle\|U^{(1)}_{W^{(1)}}\right)\nonumber\\
	&=\sum_{i=2}^{d} \sigma_iD\left(P_{W^{(i)}}\middle\|Q_{W^{(i)}}\right)\nonumber\\
	&\quad +\sigma_1 \left(D\left(P_{W^{(2)}}\middle\|M^{(2)}_{W^{(2)}}\right)+D\left(P_{W^{(1)}|W^{(2)}}\middle\|U^{(1)}_{W^{(1)}|W^{(2)}}\middle|P_{W^{(2)}}\right)\right)\label{based on chain rule2}\\
	&=\sum_{i=3}^{d} \sigma_iD\left(P_{W^{(i)}}\middle\|Q_{W^{(i)}}\right)\nonumber\\
	&\quad+(\sigma_1+\sigma_2)\left(\frac{\sigma_2}{\sigma_1+\sigma_2} D\left(P_{W^{(2)}}\middle\|Q_{W^{(2)}}\right)+\frac{\sigma_1}{\sigma_1+\sigma_2} D\left(P_{W^{(2)}}\middle\|M^{(2)}_{W^{(2)}}\right)\right)\nonumber\\
	&\quad+\sigma_1D\left(P_{W^{(1)}|W^{(2)}}\middle\|U^{(1)}_{W^{(1)}|W^{(2)}}\middle|P_{W^{(2)}}\right)\nonumber\\
	&=\sum_{i=3}^{d} \sigma_iD\left(P_{W^{(i)}}\middle\|Q_{W^{(i)}}\right)\nonumber\\
	&\quad +(\sigma_1+\sigma_2)\left(D\left(P_{W^{(2)}}\middle\|\left(M^{(2)}_{W^{(2)}},Q_{W^{(2)}}\right)_{\frac{\sigma_1}{\sigma_1+\sigma_2}}\right)+\frac{\sigma_1}{\sigma_1+\sigma_2}D_{\frac{\sigma_2}{\sigma_1+\sigma_2}}\left(P^{(2)}_{W^{(2)}}\middle\|Q_{W^{(2)}}\right)\right)\nonumber\\
	&\quad+\sigma_1D\left(P_{W^{(1)}|W^{(2)}}\middle\|U^{(1)}_{W^{(1)}|W^{(2)}}\middle|P_{W^{(2)}}\right),\label{mixing eq2}\\
	&=\left(\sum_{i=3}^{d} \sigma_iD\left(P_{W^{(i)}}\middle\|Q_{W^{(i)}}\right)+(\sigma_1+\sigma_2)D\left(P_{W^{(2)}}\middle\|\left(M^{(2)}_{W^{(2)}},Q_{W^{(2)}}\right)_{\frac{\sigma_1}{\sigma_1+\sigma_2}}\right)\right)\nonumber\\
	&\quad+\sigma_1D_{\frac{\sigma_2}{\sigma_1+\sigma_2}}\left(P^{(2)}_{W^{(2)}}\middle\|Q_{W^{(2)}}\right)
	+\sigma_1D\left(P_{W^{(1)}|W^{(2)}}\middle\|U^{(1)}_{W^{(1)}|W^{(2)}}\middle|P_{W^{(2)}}\right)\nonumber\\
	&=\underbrace{\left(\sum_{i=3}^{d} \sigma_iD\left(P_{W^{(i)}}\middle\|Q_{W^{(i)}}\right)+(\sigma_1+\sigma_2)D\left(P_{W^{(2)}}\middle\|U^{(2)}_{W^{(2)}}\right)\right)}_{B}\nonumber\\
	&\quad+\sigma_1D_{\frac{\sigma_2}{\sigma_1+\sigma_2}}\left(P^{(2)}_{W^{(2)}}\middle\|Q_{W^{(2)}}\right)
	+\sigma_1D\left(P_{W^{(1)}|W^{(2)}}\middle\|U^{(1)}_{W^{(1)}|W^{(2)}}\middle|P_{W^{(2)}}\right),\nonumber
	\end{align}
where (\ref{based on chain rule2}) follows from Lemma \ref{Chain rule T} and \eqref{mixing eq2} follows from Lemma \ref{Tilted Renyi lemma main}
.
Note that the R\'{e}nyi divergence $D_{\frac{\sigma_2}{\sigma_1+\sigma_2}}\left(P^{(2)}_{W^{(2)}}\middle\|Q_{W^{(2)}}\right)$ does not depend on $P_W$ and that expression $B$ has the same form as $A$ but with one less relative entropy. If we repeat the same procedure of using Lemmas \ref{Chain rule T} and \ref{Tilted Renyi lemma main} for expression B and so forth, we ultimately write $A$ as a sum of R\'{e}nyi divergences which do not depend on $P_W$, plus some conditional relative entropies. However, the conditional relative entropies can all be set to zero simultaneously, by taking
\begin{equation}
    \begin{cases} 
    P^{\star}_{W^{(d)}}&= U^{(d)}_{W^{(d)}}\\
    P^{\star}_{W^{(d-1)}|W^{(d)}}&= U^{(d-1)}_{W^{(d-1)}|W^{(d)}}\\
    &\vdots \\
    P^{\star}_{W^{(1)}|W^{(2)}}&= U^{(1)}_{W^{(1)}|W^{(2)}}
    \end{cases}
\end{equation}
which clearly results in the minimizer of $A$. For any distribution $P_W$, we have
\begin{align}
    P_W&=P_{W^{(1)}W^{(2)}\dots W^{(d)}}\label{functioneq}\\
       &=P_{W^{(d)}}P_{W^{(d-1)}|W^{(d)}}\dots P_{W^{(1)}|W^{(2)}}\label{Markoveq}
\end{align}
where \eqref{functioneq} follows from the fact that $W^{(2)},\dots ,W^{(d)}$ are deterministic funtions of $W=W^{(1)}$, and \eqref{Markoveq} follows from the Markov chain $W^{(1)}\leftrightarrow W^{(2)} \leftrightarrow \dots \leftrightarrow W^{(d)}$. Therefore, we deduce
\begin{equation}
\nonumber
    P^{\star}_W=U^{(1)}_{W^{(d)}}U_{W^{(d-1)}|W^{(d)}}\dots U_{W^{(1)}|W^{(2)}}.
\end{equation}
\end{proof}

\section{Proofs for Section \ref{Gaussians section}}
Let $N(\mu,\Sigma)$ denote a (multivariate) Gaussian distribution with mean $\mu$ and covariance matrix $\Sigma$.
The proof of Theorem \ref{Gaussian closedness} is based on the following well-known properties of the multivariate Gaussian distribution:
\begin{lemma}[Gaussian marginalization]\label{Gaussian marginalization}
	Assume that $N$-dimensional vector $\bf x$ has a normal distribution $N(\mu, \Sigma)$ and is partitioned as 
		$\bf x=\begin{pmatrix}
			{\bf x}_1\\
			{\bf x}_2
		\end{pmatrix} .$ Accordingly, $\mu$ and $\Sigma$ are partitioned as follows:
		$
			\mu=\begin{pmatrix}
				\mu_1\\
				\mu_2
			\end{pmatrix}$ and $ \Sigma=
			\begin{pmatrix}
				\Sigma_{11} & \Sigma_{12}\\
				\Sigma_{21} & \Sigma_{22}
			\end{pmatrix}.
		$
		Then ${\bf x}_1\sim N(\mu_1,\Sigma_{11}).$
\end{lemma}
\begin{lemma}[Gaussian conditioning]\label{Gaussian conditioning}
		Assume that $N$-dimensional vector $\bf x$ has a normal distribution $N(\mu, \Sigma)$ and is partitioned as 
		$$\bf x=\begin{pmatrix}
			{\bf x}_1\\
			{\bf x}_2
		\end{pmatrix} .$$ Accordingly, $\mu$ and $\Sigma$ are partitioned as follows:
		\begin{equation*}
			\mu=\begin{pmatrix}
				\mu_1\\
				\mu_2
			\end{pmatrix} ~ and ~ \Sigma=
			\begin{pmatrix}
				\Sigma_{11} & \Sigma_{12}\\
				\Sigma_{21} & \Sigma_{22}
			\end{pmatrix}.
		\end{equation*}
		Then the distribution of ${\bf x}_1$ conditional on ${\bf x}_2=\bf a$ is multivariate normal $({\bf x}_1|{\bf x}_2={\bf a}) \sim N(\bar{\mu},\bar{\Sigma})$ where 
		\begin{equation*}
			\bar{\mu}=\mu_1+\Sigma_{12}\Sigma_{22}^{-1}(\mathbf{a}-\mu_2)
		\end{equation*}
		and 
		\begin{equation*}
			\bar{\Sigma}=\Sigma_{11}-\Sigma_{12}\Sigma_{22}^{-1}\Sigma_{21}.
		\end{equation*}
	\end{lemma}
\begin{proof}[\bf Proof of Proposition \ref{Gaussian concatenation}]
Let $U_{W_1}^{(1)}=N(\mu,Q^{-1})$ (where $Q$ is the precision matrix) and  $U_{W_1,W_2}^{(2)}=N(\bar{\mu},{\bar{Q}}^{-1})$ where $\bar{\mu}=[\mu_1,\mu_2]^T$ and 
\[\bar{Q}=
\begin{pmatrix}
	A & B \\
      B^T & D \\
\end{pmatrix}.
  \]
    Based on Lemma \ref{Gaussian conditioning}, we have 
    \[U_{W_2|W_1=w_1}^{(2)}=N(\bar{\mu}_2-D^{-1}B^T(w_1-\bar{\mu}_1),D).\]
    Therefore, 
    \begin{align*}
      P_{W_1W_2}(w_1,w_2)&=U_{W_1}^{(1)}U_{W_2|W_1}^{(2)}\nonumber\\
      &\propto  \exp\bigg(-\frac{1}{2}(w_1-\mu)^TQ(w_1-\mu)\\
      &\quad-\frac{1}{2}\big(w_2-\bar{\mu}_2+D^{-1}B^T(w_1-\bar{\mu}_1)\big)^TD(w_2-\bar{\mu}_2+D^{-1} B^T(w_1-\bar{\mu}_1))\bigg)\\
      &=\exp\biggl(-\frac{1}{2}\bigg[(w_1-\mu)^TQ(w_1-\mu)+(w_1-\bar{\mu}_1)^TBD^{-1}DD^{-1}B^T(w_1-\bar{\mu}_1)\\
      &\quad+(w_1-\bar{\mu}_1)^TBD^{-1}D(w_2-\bar{\mu}_2)+(w_2-\bar{\mu}_2)^TDD^{-1}B^T(w_1-\bar{\mu}_1)\\
      &\quad+(w_2-\bar{\mu}_2)^TD(w_2-\bar{\mu}_2)\bigg]\biggr).
          \end{align*}
          Thus, $P_{W_1W_2}$ is a multivariate Gaussian distribution with precision matrix
  \begin{equation}
  	\nonumber
  	\widehat{Q}=
  \begin{pmatrix}
  	Q+BD^{-1}B^T & B \\
      B^T & D \\
  \end{pmatrix}.
  \end{equation}
  Its mean can be readily derived as well.
  \end{proof}

\section{Proofs for Section \ref{neural network section}}

\begin{proof}[\bf Proof of Theorem \ref{CMI generalization deep nets theorem}]
    For all $1\leq i\leq d$, let $$L_{\mu}(W_1,\dots,W_i)\triangleq \E[|h_i(x)-y|^2_2],$$ $$L_{S}(W_1,\dots,W_i)\triangleq \frac{1}{n}\sum_{i=1}^n[|h_i(x_i)-y_i|^2_2]$$ and $$\mathrm{gen}(W_1,\dots,W_i)\triangleq L_{\mu}(W_1,\dots,W_i)-L_{S}(W_1,\dots,W_i).$$ We can write
    \begin{align}
        \E[L_{\mu}(W)]-\E[L_S(W)]&=\E[\gen(W_1,\dots,W_d)]\nonumber\\
        &=\E(\gen(W_1,\dots,W_{d})-\gen(W_1,\dots,W_{d-1}))\nonumber\\
        &\quad+\dots+\E(\gen(W_1,W_2)-\gen(W_1))+\E(\gen(W_1))\nonumber\\
        &\leq \frac{C}{d\sqrt{n}}\sum_{i=1}^{d}\sqrt{I\left(S;W_1,\dots,W_{d-i+1}\right)}\label{CMI boundd}\\
        &=\frac{C}{d\sqrt{n}}\inf_{\gamma, Q_W}\sum_{i=1}^{d}\left(\gamma_i D\left(P_{W_1\dots W_{d-i+1}|S}\middle\|Q_{W_1\dots W_{d-i+1}}\middle|P_S\right)+\frac{1}{4\gamma_i} \right),\label{variational}
    \end{align}
    where $C=2(eR)^2$, \eqref{CMI boundd} follows from Lemma \ref{chaining link distance} and the mutual information bound \cite{Russo, xu2017information}, similar to the technique of \cite{asadi2018chaining}, and \eqref{variational} follows from removing the square root and replacing mutual information with conditional relative entropy, variationally.
\end{proof}
\begin{proof}[\bf Proof of Theorem \ref{DPG thoerem}] Based on induction on $i$ and the triangle inequality, we have $|h_{i-1}|_2\leq \exp\left(\frac{i-1}{d}\right)|x|_2\leq e|x|_2.$ In particular, the output of the network is bounded as $|h_d(x)|_2\leq eR$. Therefore, on the one hand, the single-scale generalization bound has the following form:
\begin{align}
\E[L_{\mu}(W)]&\leq \E[L_S(W)]+\frac{C}{\sqrt{n}}\sqrt{I(S;W_1,\dots,W_d)}\nonumber\\
            &=\E[L_S(W)]+\frac{C}{\sqrt{n}}\inf_{\hat{\gamma}, Q_W}\left(\hat{\gamma}D(P_{W_1\dots W_d|S}\|Q_{W_1,\dots,W_d}|P_S)+\frac{1}{4\hat{\gamma}}\right). \label{gibbsupper}
\end{align}
For a given $\widehat{Q}_W$ and $\hat{\gamma}$, let $P^{\rm{Gibbs}}_{W|S}$ be the single-scale Gibbs posterior which minimizes the right side of \eqref{gibbsupper}. With a similar technique to \cite{xu2017information}, we can write
\begin{align}
    \risk\left(P^{\rm{Gibbs}}_{W|S} \right) & = \E[L_{\mu}(W)]\nonumber\\
    &\leq \E[L_S(W)]+\frac{C}{\sqrt{n}}\left(\hat{\gamma}D(P_{W_1\dots W_d|S}\|{Q}_{W_1,\dots,W_d}|P_S)+\frac{1}{4\hat{\gamma}}\right)\nonumber\\
    & \leq \E[L_S(\bar{W})]+\frac{C}{\sqrt{n}}\left(\hat{\gamma}D(\widehat{Q}_{{W}_1\dots {W}_d}\|{Q}_{W_1,\dots,W_d}|P_S)+\frac{1}{4\hat{\gamma}}\right)\label{based on prev}\\
    & = \E[L_{\mu}(\bar{W})] + \frac{C}{\sqrt{n}}\left(\hat{\gamma}D(\widehat{Q}_W\|{Q}_{W})+\frac{1}{4\hat{\gamma}}\right),\label{excess_risk1}
\end{align}
where $\bar{W}\sim \widehat{Q}_W$ and \eqref{based on prev} follows from the fact that $P^{\rm{Gibbs}}_{W|S}$ minimizes \eqref{gibbsupper}.
The excess risk \eqref{excess_risk1} is minimized by taking $\hat{\gamma}\gets 1/\sqrt{4D(\widehat{Q}_W\|{Q}_{W})}$. For such $\hat{\gamma}$, the Gibbs posterior satisfies the following bound on its population risk:
\begin{equation}
   \risk\left(P^{\rm{Gibbs}}_{W|S} \right)= \E[L_{\mu}(W)]\leq \E[L_{\mu}(\bar{W})]+\frac{C}{\sqrt{n}}\sqrt{D(\widehat{Q}_W\|{Q}_{W})}. \label{Excess_optimal1}
\end{equation}
On the other hand, based on Theorem \ref{CMI generalization deep nets theorem}, for the multiscale Gibbs posterior $P^{\star}_{W|S}$ given $\gamma$ and $\widehat{Q}_W$, we can write 
\begin{align}
\risk\left(P^{\star}_{W|S} \right) & = \E[L_{\mu}(W)]\nonumber\\
    & \leq \E\left[ L_S(W)\right]+ \frac{C}{d\sqrt{n}}\sum_{i=1}^{d}\left(\gamma_i D\left(P_{W_1\dots W_{d-i+1}|S}\middle\|Q_{W_1\dots W_{d-i+1}}\middle|P_S\right)+\frac{1}{4\gamma_i} \right)\nonumber \\
    & \leq \E[L_S(\bar{W})] + \frac{C}{d\sqrt{n}}\sum_{i=1}^{d}\left(\gamma_i D\left(\widehat{Q}_{{W}_1\dots {W}_{d-i+1}}\middle\|Q_{W_1\dots W_{d-i+1}}\middle|P_S\right)+\frac{1}{4\gamma_i} \right)\nonumber\\
    & =\E[L_{\mu}(\bar{W})] + \frac{C}{d\sqrt{n}}\sum_{i=1}^{d}\left(\gamma_i D\left(\widehat{Q}_{{W}_1\dots {W}_{d-i+1}}\middle\|Q_{W_1\dots W_{d-i+1}}\right)+\frac{1}{4\gamma_i} \right)\label{excess_risk2}
\end{align}
The excess risk \eqref{excess_risk2} is minimized by taking ${\gamma_i}\gets 1/\sqrt{4D(\widehat{Q}_{W_1\dots W_{d-i+1}}\|{Q}_{W_1\dots W_{d-i+1}})}$ for all $1\leq i \leq d$. For such ${\gamma}$, the multiscale Gibbs posterior satisfies the following bound on its population risk:
\begin{equation}
   \risk\left(P^{\star}_{W|S} \right)= \E[L_{\mu}(W)]\leq \E[L_{\mu}(\bar{W})]+\frac{C}{d\sqrt{n}}\sum_{i=1}^d\sqrt{D\left(\widehat{Q}_{W^{(i)}}\|{Q}_{W^{(i)}}\right)}. \label{Excess_optimal2}
\end{equation}
The difference between the right sides of \eqref{Excess_optimal1} and \eqref{Excess_optimal2} is equal to 
\begin{equation}
    \frac{C}{d\sqrt{n}}\left(d\sqrt{D(\widehat{Q}_W\|{Q}_{W})} - \sum_{i=1}^d\sqrt{D\left(\widehat{Q}_{W^{(i)}}\|{Q}_{W^{(i)}}\right)}  \right)
    = \frac{C}{d\sqrt{n}}\sum_{i=1}^d \mathrm{DPG}(i).
\end{equation}
\end{proof}
An example of excess risk bound when the synaptic weights take continuous values is as follows. 
\begin{theorem}[Excess risk bound] Let $\widehat{w}=({\widehat{w}}_1,\dots,{\widehat{w}}_d)$ denote a set of weight parameters which achieve the minimum population risk among the whole hypothesis set (i.e., $L_{\mu}(\hat{w})=\inf_{w\in\mathcal{W}}L_{\mu}(w))$. Let $\mathcal{B}^{\epsilon}\triangleq \left\{w\in \mathcal{W}:  \|w_i-{\widehat{w}}_i\|_2\leq \epsilon \textrm{~ for all ~}1\leq i \leq d\right\}$, $\rho(\epsilon)\triangleq \sup_{w\in \mathcal{B}^{\epsilon}}\left\{L_{\mu}(w)-L_{\mu}(\widehat{w})\right\}$, and $B^{(\epsilon)}$ be the uniform distribution on $\mathcal{B}^{\epsilon}$.  For any prior distribution $Q_W$, there exists $\gamma=(\gamma_1,\dots,\gamma_d)$ such that
    \begin{equation}\label{CMI excess risk arbitrary ineq optimal}
        \E\left[ L_{\mu}(W)\right] - \inf_{w\in\W}L_{\mu}(w)\leq  \inf_{\epsilon>0}\left\{\rho(\epsilon)+\frac{C}{d\sqrt{n}}\sum_{i=1}^d \sqrt{ D\left(B^{(\epsilon)}_{W_1\dots W_i}\middle\|Q_{W_1\dots W_i}\right)}\right\}.
\end{equation}
\end{theorem}
\begin{proof}
Follows from \eqref{Excess_optimal2} by choosing $\widehat{Q}_W=B^{(\epsilon)}_W$ and taking the infimum over $\epsilon$.
\end{proof} 

\subsection{Teacher-Student example}
Let $d'\triangleq d/M$. By neglecting the value of $\log \frac{1}{q_1}$ with respect to $\log \frac{1}{q_2}$, we have 
\begin{align*}
    \sum_{i=1}^d\mathrm{DPG}(i)&\approx\sqrt{\log \frac{1}{q_2}}\left(d\sqrt{d'}- (\sqrt{1}+\sqrt{2}+\dots+\sqrt{d'}) \right)\\
                   &\approx \sqrt{\log \frac{1}{q_2}}\left(d\sqrt{d'}- \int_0^{d'}\sqrt{x}\mathrm{d}x \right)\\
                   &= \sqrt{\log \frac{1}{q_2}}\left(d\sqrt{d'}- \frac{2}{3}(d^{'})^{\frac{3}{2}} \right)\\
                   &=\left(\sqrt{\log \frac{1}{q_2}}\right)d^{\frac{3}{2}}\left( \frac{M-\frac{2}{3}}{M^{\frac{3}{2}}} \right).
\end{align*}

\subsection{Experiment}
\begin{figure}
    \centering
    \includegraphics[width=0.7\textwidth]{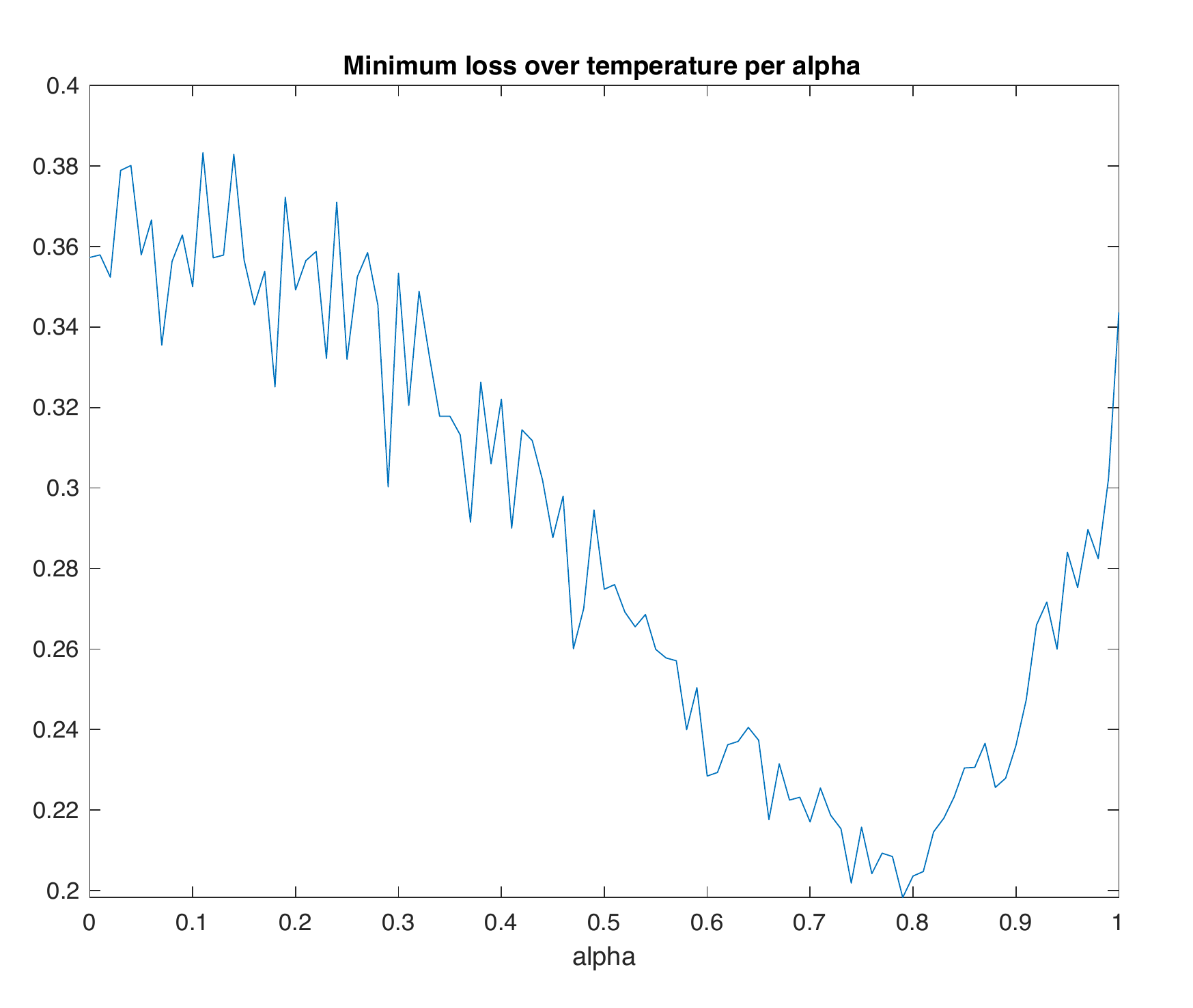}
    \caption{Minimum loss over temperature per $\alpha$. The vertical axis denotes the population risk of a hypothesis randomly chosen by the multiscale Gibbs posterior. Note that $\alpha=0$ corresponds to the single-scale Gibbs distribution, and $\alpha\to 1$ corresponds to learning with random features. The prior distribution has variance $5\times 10^{-4}$.}
    \label{fig2}
\end{figure}
We let the teacher and student networks have width $m=10$. The teacher network had depth $d'=2$ and the student network had depth $d=4$. We let the training set to have size $n=30$ and the instances be random vectors drawn from an i.i.d Gaussian distribution with unit variance. The weights of the teacher networks were drawn randomly from a Gaussian distribution with independent entries on the synapses each with variances $0.1$. We use the Gauss--Newton approximation to the Hessian along with the gradient at the origin to obtain a Gaussian approximation to the initial Gibbs posterior. The range of temperature $\sigma_1$ was chosen as $[10^{-9.5},10^{-2.5}]$. The prior distribution $Q_{W_1\dots W_d}$ was let to be an i.i.d. Gaussian distribution with zero mean and with variance $5\times 10^{-5}$.
For variance $5\times 10^{-4}$, the simulation is repeated with the result given in Figure \ref{fig2}.

We use the following lemma to sample the weights from Gaussian distribution:   
\begin{lemma}[Gaussian sampling]\label{Gaussian sampling lemma}
	Let $\mathbf{z}\sim N(\mathbf{0},I)$. Assume that $\Sigma$ is positive definite matrix and let its Cholesky decomposition be $\Sigma=CC^T$. Let $\mathbf{x}=\mu+C\mathbf{z}$. Then $\mathbf{x}\sim N(\mu,\Sigma)$. 
	\end{lemma}
	
	\bibliographystyle{unsrt}
\bibliography{MyBibliography.bib}
\end{document}